\documentclass[10pt]{article}

\usepackage[accepted]{tmlr}
\def\month{04}
\def\year{2026}
\newcommand{\openreview}{\href{https://openreview.net/forum?id=5771}{https://openreview.net/forum?id=5771}}

\usepackage[utf8]{inputenc} 
\usepackage[T1]{fontenc}    
\usepackage{hyperref}       
\usepackage{url}            
\usepackage{booktabs}       
\usepackage{amsfonts}       
\usepackage{nicefrac}       
\usepackage{microtype}      
\usepackage{xcolor}         
\usepackage{subcaption}
\usepackage[linesnumbered,ruled]{algorithm2e}
\usepackage{wrapfig}  

\usepackage[utf8]{inputenc} 
\usepackage[T1]{fontenc}    
\usepackage{hyperref}       
\usepackage{url}            
\usepackage{booktabs}       
\usepackage{amsfonts}       
\usepackage{nicefrac}       
\usepackage{microtype}      
\usepackage{xcolor}         
\usepackage{multirow}       

\usepackage{amsmath}
\usepackage{amssymb}
\usepackage{mathtools}
\usepackage{amsthm}

\usepackage{thm-restate}

\usepackage{color-edits}
\addauthor{ev}{blue}

\theoremstyle{plain}
\newtheorem{theorem}{Theorem}[section]

\theoremstyle{definition}
\newtheorem{definition}[theorem]{Definition}

\theoremstyle{remark}

\newenvironment{hproof}{%
  \proof}{\endproof}

\title{Let's Measure Information Step-by-Step: \\ AI-Based Evaluation Beyond Vibes}

\author{\name Zachary Robertson \email zroberts@cs.stanford.edu \\
        \name Sanmi Koyejo \email sanmi@cs.stanford.edu \\
        \addr Department of Computer Science\\
        Stanford University}
\begin{document}

\maketitle

\begin{abstract}
We evaluate artificial intelligence (AI) systems without ground truth by exploiting a link between strategic gaming and information loss. Building on established information theory, we analyze which mechanisms resist adversarial manipulation. This motivates mutual evaluation, where the overseer is treated as a strategic player estimating mutual information by prompting, making truthful agent reporting an optimal strategy. We show that certain f-divergences, such as total variation distance (TVD), maintain polynomial guarantees under attack, building on an established exponential barrier for estimating mutual information (MI) in worst-case certification settings. Under adversarial attacks, TVD-MI maintains effectiveness (area under the curve 0.70--0.77) while other approaches can decay toward chance, demonstrating that prompting the same system for information relationships rather than quality judgments can improve robustness. The mechanisms decompose pairwise evaluations into reliable item-level detection scores without ground truth, addressing a key limitation of standard peer prediction. \textit{Pre-registration: \url{https://osf.io/c7pum}.}
\end{abstract}

\section{Introduction}

When artificial intelligence (AI) systems evaluate other AI systems without ground truth, we face a fundamental challenge: how can we distinguish truthful information sharing from strategic manipulation? In scientific peer review, technical analysis, and other specialized tasks, human overseers often lack the expertise to verify AI-generated content directly. While traditional methods compare outputs to known correct answers, this approach fails when such verification is infeasible or when the AI system possesses knowledge beyond human oversight capabilities.

Current approaches to AI evaluation face significant limitations. Human experts face practical limits when assessing quality at scale, especially given expertise gaps. Standard automated metrics fail when reference outputs don't exist for novel tasks. Recent methods using large language models (LLMs) as judges \citep{zheng2023judging} can exhibit bias and, as we demonstrate, can be manipulated to invert quality rankings entirely. These limitations become critical as AI systems increasingly evaluate other AI systems, creating potential evaluation loops disconnected from ground truth. We must develop fundamentally different principles that do not rely on direct quality assessment.

We propose a solution: instead of asking "which output is better?"—a question adversaries can manipulate—we ask "do these outputs share information about the same source?"—a relationship protected by the data processing inequality. This inequality states any strategic manipulation of content necessarily reduces mutual information between responses. When we measure these information relationships, we can implement mechanisms with formal gaming-resistance guarantees.

Our approach connects two previously separate frameworks through information theory. From mechanism design, we recognize evaluation as a game where agents strategically manipulate outputs. From the Eliciting Latent Knowledge (ELK) framework \citep{christiano2022eliciting}, we recognize that the core challenge is information asymmetry: AI agents possess knowledge we cannot directly verify. We combine these perspectives to formalize evaluation as an information elicitation game where truthful reporting can be incentivized by designing scoring rules based on mutual information between agent responses.

\textbf{Our Results.} Our results confirm that information-theoretic mechanisms are more robust compared to quality-based evaluation. Figure \ref{fig:latent-knowledge} shows our setup. We extend \citet{mcallester2020formal} to prove bounded f-divergences resist adversarial tampering (Theorem~\ref{thm:wclb}) and validate across 10 domains:

\begin{enumerate}
\item \textbf{Mechanisms detect manipulation where judges fail.} Information-theoretic mechanisms consistently score faithful content above problematic content. LLM judges require references and multi-pair aggregation to be competitive. 

\item \textbf{Item-level detection without ground truth.} Information-theoretic mechanisms achieve AUC 0.64-0.77 when distinguishing faithful--faithful from faithful--problematic pairs.

\item \textbf{Gaming resistance persists under attack.} Under adversarial transformations, our proposed total-variation distance mutual information (TVD-MI) mechanism maintains effectiveness (AUC > 0.70) while judges degrade to near-random performance (0.54-0.67).
\end{enumerate}

These findings suggest an alternative approach to AI evaluation that uses identical models but provides formal guarantees, becoming important as AI systems increasingly evaluate AI-generated content without human verification.

\textbf{Roadmap.} Section~\ref{sec:framework} introduces our information-theoretic framework and develops the distribution-free sample-complexity analysis. Section~\ref{sec:certification} describes the TVD-MI implementation via a binary ``same source?'' critic. Section~\ref{sec:study-setup} outlines the experimental setup. Section~\ref{sec:findings} presents results across three domains, including adversarial robustness stress tests.

\begin{figure}[t]
  \centering
  \begin{subfigure}[b]{0.49\linewidth}
    \includegraphics[width=\linewidth]{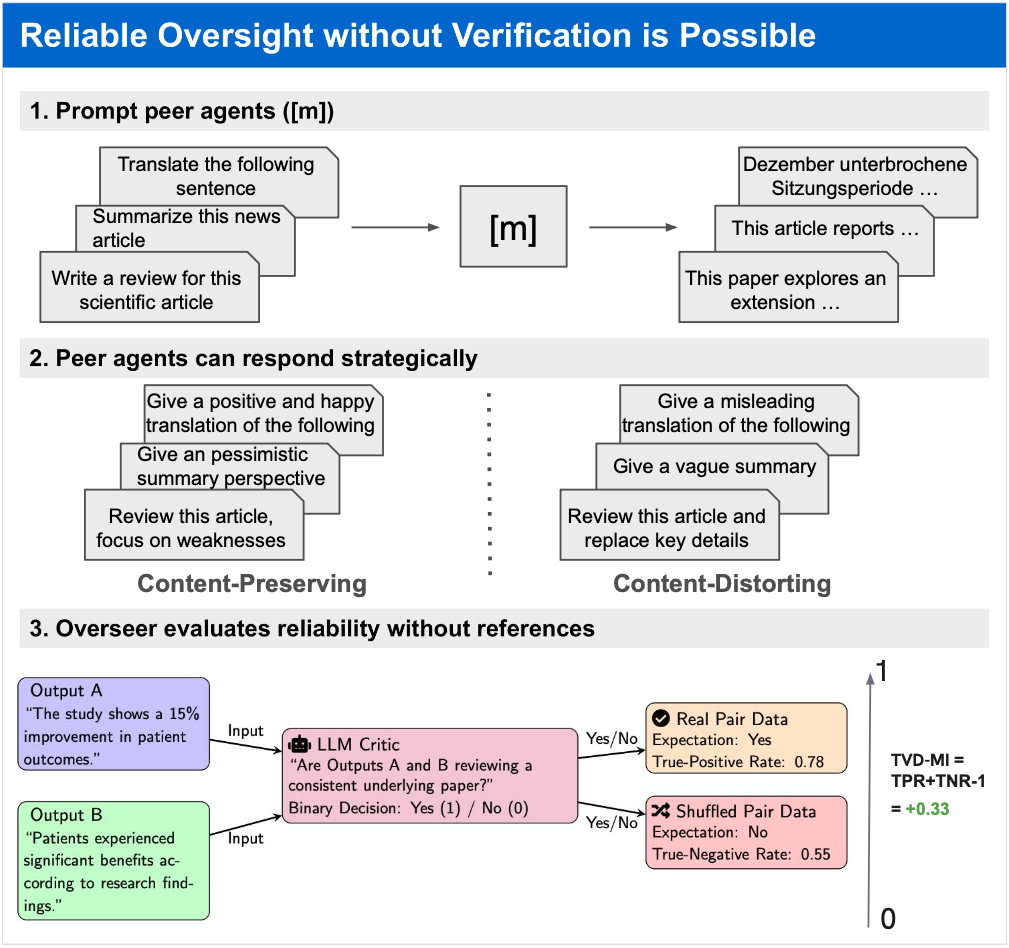}
    \label{fig:it-limits}
  \end{subfigure}
  \hfill
  \begin{subfigure}[b]{0.49\linewidth}
    \includegraphics[width=\linewidth]{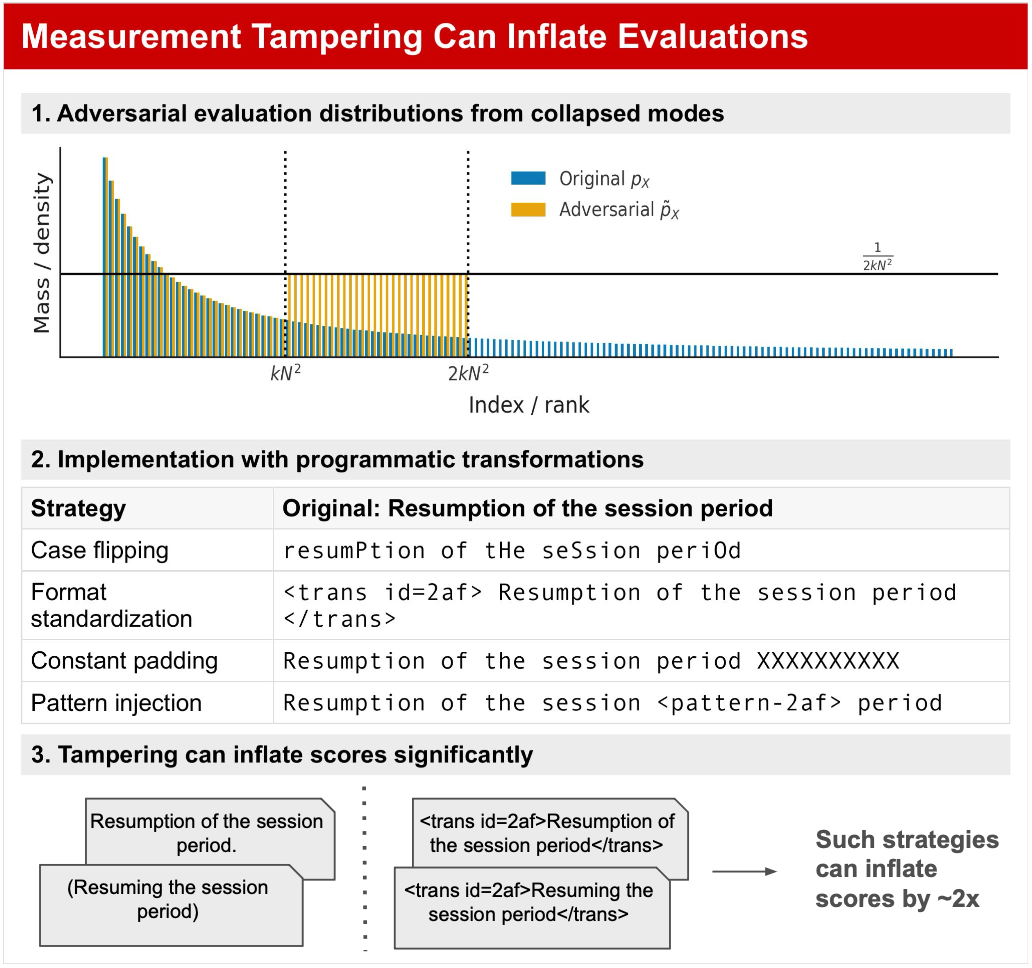}
    \label{fig:robustness}
  \end{subfigure}
  \vspace{-4pt}
  \caption{\textbf{Overview of our study:} we study mutual mechanisms that are robust to strategic reporting. \textbf{Left} (Section \ref{sec:findings}): Multiple AI agents generate responses to the same source. Without reference answers, how can we identify quality? \textbf{Right} (Section \ref{sec:limits}): A theoretical visualization of an agent manipulating its response distributions. We demonstrate this with real attacks that introduce artificial uniformity that maintains information, but can collapse the evaluation distribution and distort scores.}
  \label{fig:latent-knowledge}
\end{figure}

\section{Background and Related Work}
\label{section:related_work}

\textbf{LLM Evaluation and Oversight.} LLM-based evaluations can carry biases, especially when evaluators share architecture or training data with evaluated models \citep{zheng2023judging, chen2024humans}. RLHF and Constitutional AI attempt to mitigate these biases through structured human oversight \citep{christiano2017deep, bai2022constitutional}, while debate and recursive reward modeling provide alternative frameworks \citep{irving2018ai, bowman2022measuring}. These methods typically do not consider evaluator incentives explicitly. We frame evaluation as mechanism design with explicit incentive analysis. Our empirical findings confirm and extend these concerns, showing that LLM judges can exhibit bias and mis-rank quality judgments.

\textbf{Eliciting Latent Knowledge (ELK).} ELK refers to methods designed to induce truthful reporting from models rather than outputs optimized solely for approval \citep{christiano2022eliciting}. Existing ELK techniques probe internal model representations to interpret latent knowledge \citep{burns2022discovering, marks2023geometry}. Our work formulates ELK as a black-box peer prediction mechanism, focusing on strategic gaming robustness without requiring white-box model access. This is motivated by findings that LLM hidden states encode truthfulness-related variables that are linearly separable across diverse tasks \citep{marks2023geometry}, allowing us to treat model outputs as strategic transformations of latent knowledge states.

\textbf{Peer Prediction and Strategy-Proofness.} Peer prediction mechanisms incentivize truthful reporting without verification \citep{prelec2004bayesian}. Recent advancements have introduced information-theoretic frameworks \citep{kong2018water, schoenebeck2020learning} and LLM-specific adaptations such as ElicitationGPT \citep{wu2024elicitationgpt}, GPPM \citep{lu2024eliciting}, and GEM \citep{xu2024benchmarking} for model benchmarking. However, these methods separate evaluation into pre-processing and scoring, which confounds formal analysis of adversarial settings. Our approach explicitly models overseer incentives, and uses a single evaluation model to score all agent outputs, eliminating confounds from model-specific biases without requiring access to probability estimates from the underlying model (see Table \ref{tab:pp-compare}).

\textbf{Connections to Contrastive Learning.} Our f-MI mechanisms parallel contrastive learning objectives \citep{chen2020simple}, where distinguishing positive pairs (same source) from negative pairs (different sources) mirrors our TVD-MI critic's task. This connection suggests the critic could be further trained using self-supervised learning. For measurement integrity, we extend adversarial MI estimation bounds \citep{mcallester2020formal} to characterize statistical limits, advancing prior theoretical results by integrating adversarial robustness concerns directly into incentive design.

\begin{table}[t]
\centering
\caption{Comparison of recent peer-prediction mechanisms for LLM evaluation.\label{tab:pp-compare}}
\footnotesize
\begin{tabular}{@{}p{8.0cm}cccl@{}}
\toprule
\textbf{Method} & \textbf{Overseer} & \textbf{Reference} & \textbf{Distribution-free} & \textbf{Black-box} \\
& \textbf{modeled} & \textbf{free} & \textbf{analysis} & \textbf{sufficient} \\
\midrule
ElicitationGPT \citep{wu2024elicitationgpt} & No & No & No & Yes \\
GEM \citep{xu2024benchmarking} & No & No & No & No \\
GPPM \citep{lu2024eliciting} & No & Yes & No & No \\
TVD-MI CoT mechanism \textbf{(ours)} & Yes & Yes & Yes & Yes \\
\bottomrule
\vspace{1mm}
{\scriptsize Note: Black-box sufficient means no token probabilities required.}
\end{tabular}
\end{table}

\paragraph{What are our contributions?} Prior peer prediction work typically assumes honest reporting; we study adversarial tampering against the overseer. Our main result (Theorem \ref{thm:wclb}) gives finite-sample robustness bounds for $f$-MI mechanisms under mode-collapse attacks, showing that certain measures such as TVD can certify large scores with few samples whereas other ones can require increasing samples for each additional bit in the worst case. Moreover, because TVD-MI is naturally evaluated as a chain-of-thought (CoT) mechanism it can be implemented with any LLM API, whereas probability based methods require specific features that are inconsistently supported across providers \citep{cai2025you}.

\section{Theoretical Framework}
\label{sec:framework}

This section develops our theoretical framework and presents Theorem~\ref{thm:wclb}, a distribution-free sample-complexity bound under adversarial distribution manipulation. The result generalizes the indistinguishability construction of \citep{mcallester2020formal} to general $f$-divergences, revealing a fundamental separation between bounded and unbounded choices. We use this framework to formalize the overseer's limited information state in a peer-prediction setting and to characterize statistical limits on detecting strategic manipulation. We then describe a practical implementation via a variational chain-of-thought procedure that preserves item-level interpretability.

\subsection{Mutual Evaluation Games}
\label{sec:info-elicitation}
We formalize mutual evaluation as a type of information elicitation game where truthful reporting emerges from aligned incentives. Agents and overseers interact when agents report information to an overseer who must assess quality without ground truth. This framework captures an important challenge of AI oversight: distinguishing truthful information sharing from strategic manipulation when verification is not available.

Agents $i$ and $j$ receive private signals $Y_i, Y_j$ from their environment—documents to summarize, papers to review, or text to translate. Each agent transforms their signal using a reporting strategy $\theta_i: \mathcal{Y} \to \Delta(\mathcal{Y})$, potentially adding randomness to their report. The overseer must design rules that incentivize agents to report truthfully despite being unable to verify content directly.

Our approach leverages information-theoretic measures that quantify statistical dependencies between reports. When agents truthfully report about the same source, their outputs share genuine information. When agents manipulate strategically, they disrupt these patterns, creating detectable distortions measurable through $f$-divergences.

\begin{definition}[$f$-Mutual Information]
\label{regMI}
$f$-divergences quantify the information shared between reports in a way that resists manipulation. Given random variables $X,Y$ with joint distribution $P_{XY}$, the $f$-mutual information is:
\begin{equation}
I_f(X;Y)=D_f(P_{XY}\Vert P_X\otimes P_Y) := \sum_{i,j} P_{X}(i) P_{Y}(j) \cdot f\left(\frac{P_{XY}(i,j)}{P_X(i)\cdot P_Y(j)}\right),
\end{equation}
where $f$ is convex, with $f(1) = 0$ and $f(0) < \infty$, nowhere constant. 
\end{definition}

This family includes Shannon mutual information ($f(t) = t\log t$) and total variation distance mutual information ($f(t) = \frac{1}{2} |t-1|$). The choice of $f$ determines not only statistical efficiency but also sample-complexity under adversarial distributions. To understand this we first describe the role of the overseer. 

\begin{definition}[Empirical Joint Type]
Given a sample $S=\{(x_1,y_1),\dots,(x_N,y_N)\}$, let $\mathcal{T}(S)(i,j)$ be the number of occurrences of pair $(i,j)$ in $S$. 
The \emph{empirical joint type} is the contingency table $\mathcal{T}(S)$ modulo independent permutations of the row and column labels. 
Any estimator depending only on this statistic is called type-based.
\end{definition}

\paragraph{The Overseer as an Agent.}
The overseer is a computational procedure that observes a finite sample of paired reports and outputs a score. We emphasize two objects: (i) the overseer’s \emph{private signal}, which we take to be the empirical type, and (ii) a fixed \emph{decision rule} $r$ chosen before seeing the realized sample. Section~\ref{sec:certification} instantiates $r$ using an LLM prompting policy (Figure \ref{fig:latent-knowledge} (left)); the resulting test statistics yield certified lower bounds on $f$-mutual information.

The empirical joint type $\mathcal{T}(S)$ formalizes the overseer’s finite-sample information state. The reasoning strategy $r$ maps $\mathcal{T}(S)$ to a categorical judgment, and $\mathcal{T}(S)$ also underlies the indistinguishability construction in Theorem~\ref{thm:wclb}, where alternative joint laws are matched at the level of empirical types.

\paragraph{Game Structure.}
An agent-overseer mutual evaluation game proceeds as:
\begin{enumerate}
\item \textbf{Nature} generates a source and distributes $n$ joint signals $(Y_i^{(n)}; Y_j^{(n)}) \sim P^{(n)}_{ij}$ to agents
\item \textbf{Agents} apply strategies $\theta_i, \theta_j$ generating a multi-set of reports $(\theta_i^{(n)}; \theta_j^{(n)}) := S^{(n)}_{ij}$
\item \textbf{Overseer} applies a reasoning strategy $r$ over $\mathcal{T}(S_{ij}^{(n)})$ producing an estimate $\hat I_f^r(\mathcal{T}(S_{ij}^{(n)}))$  
\item \textbf{Mechanism} pays all participants based on certified (Section ~\ref{sec:certification}) $f$-MI scores:
\begin{equation}
u_i = \sum_{j \neq i} \widehat{I}_f^r(\mathcal{T}(S_{ij}^{(n)})), \quad u_{\text{overseer}} = \sum_{j} u_j
\end{equation}
\end{enumerate}

This payment structure creates aligned incentives: agents maximize scores by preserving information, while the overseer maximizes by accurately estimating a lower-bound on mutual information. Unlike traditional evaluation where judges might exhibit bias, our mechanism ensures truthful estimation is the overseer’s best response. We implement our prompting-based estimator $\hat I_f^r(\cdot)$ for TVD in Section \ref{sec:certification}. 

\subsection{The Dual Nature: Incentives and Quality}
\label{sec:incentiveQuality}
Our mechanisms serve a dual purpose. The \emph{design objective} incentivizes truthful reporting through strategic robustness. The \emph{validation method} establishes correlation with quality metrics where ground truth exists. The data processing inequality ensures that strategic manipulation can only degrade mutual information. Thus, in applications we interpret the mechanism primarily as an information-preservation signal rather than a universal quality metric. When agents attempt to game the mechanism by distorting their reports, they simultaneously reduce the mutual information between their response and the source (what we measure) and degrade the quality of their output (what we care about).

\paragraph{Connection to Classical Reliability Measures.}
Our focus on TVD-MI generalizes classical inter-rater reliability measures to high-dimensional settings. As shown in Appendix~\ref{appendix:reliability}, TVD-MI provides a lower bound for Cohen's $\kappa$ normalized by chance agreement. Moreover, for binary classification tasks, TVD-MI directly relates to Youden's \citeyearpar{youden1950index} J statistic ($\text{TPR} + \text{TNR} - 1$), which measures informativeness \citep{powers2012problem}. This connection explains why our mechanisms successfully produce AUC scores (Section~\ref{sec:decomposability}). All three measures ($\kappa$, AUC, informativeness) quantify the same underlying information-theoretic relationship from different perspectives.

\paragraph{Gaming-Resistance $\Rightarrow$ Data-Processing Inequality.}
We formalize gaming-resistance over strategies that post-process only $Y_i$ (stochastic channels $\theta_i: Y_i\!\to\!\Delta(\mathcal Y_i)$; no shared coins), requiring that expected score cannot increase \citep{kong2019information}. If scores are functions of statistical dependence (e.g., $f$-mutual information), then any post-processing $\theta_i(Y_i)$ yields a Markov chain $Y_j \to Y_i \to \theta_i(Y_i)$ and the data processing inequality (DPI) for $f$-divergences gives
\[
I_f\!\big(\theta_i(Y_i)\,;\,Y_j\big)\;\le\; I_f\!\big(Y_i\,;\,Y_j\big).
\]
Hence gaming-resistance holds directly from DPI \citep{sason2016f}. These connections explain why mechanisms designed for gaming-resistance also identify high-quality outputs: both properties emerge from information preservation. The equivalence is not universal, and it can fail under adaptive strategies that add shared but unhelpful overlap rather than merely post-processing an agent's signal. We therefore treat empirical quality correlations as validation evidence for the tested domains, not as an assumption-free guarantee.

\paragraph{From Gaming-Resistance to Incentive Compatibility.}
Because DPI holds regardless of the peer’s strategy $\theta_j$, truthful reporting (the identity channel) weakly dominates any post-processing of $Y_i$. When payments are an affine function of $I_f$ (see Section~\ref{sec:info-elicitation}), the agent’s expected utility is maximized by reporting truthfully for all $\theta_j$. Thus, gaming-resistance implies dominant-strategy incentive compatibility (DSIC) \emph{within the class of strategies that are functions of the agent's signal}. Strictness follows under strictly convex $f$ and non-degenerate signals (identity is then the unique maximizer).

\subsection{Statistical Limits for Gaming-Resistance}
\label{sec:limits}

In this game, the overseer estimates $I_f(X;Y)$ from finite samples. Without knowledge of the response distribution, any estimator faces a worst-case adversary who can manipulate the distribution to minimize information while remaining consistent with observed samples. Our main robustness result, Theorem~\ref{thm:wclb}, upper-bounds the largest reliable lower bound any estimator can achieve. This yields a distribution-free sample-complexity bound under adversarial manipulation, extending McAllester and Stratos~(2020)’s indistinguishability construction from Shannon mutual information to general $f$-divergences. In the worst case, bounded, piecewise-linear $f$ (e.g., total variation) permit certification ceilings growing polynomially with sample size, whereas unbounded, super-linear $f$ (e.g., Kullback--Leibler) show only logarithmic growth, requiring exponentially more samples to certify extra nats. This separation motivates TVD-MI, estimated with a binary “real or shuffled pair?” critic (Fig.~\ref{fig:latent-knowledge}, left). For a fixed critic, a bounded test statistic concentrates as $O(1/\sqrt{N})$ \citep{BoucheronLugosiMassart2013}; Theorem~\ref{thm:wclb} instead characterizes the estimator-independent certification limit.

\paragraph{Why a ceiling bound (and how it relates to prior $f$-divergence estimation).}
Theorem~\ref{thm:wclb} is a \emph{distribution-free certification ceiling}: it upper-bounds the largest value that any $(1-\delta)$ \emph{lower confidence bound} can safely certify from $N$ samples when the estimator depends only on the empirical type. Prior work (e.g., \citet{schoenebeck2020learning}) studies \emph{estimation} sample-complexity of $f$-divergences, but assumes a functional class that contains an efficient or near-optimal scoring rule. Our result instead considers a worst-case adversary that selects joint laws indistinguishable at the level of empirical types at finite $N$, and therefore bounds how much information can ever be \emph{reliably certified}. For bounded $f$ (e.g., TV), this ceiling approaches the maximum possible information up to vanishing error, whereas for unbounded $f$ (e.g., Shannon/KL) it grows only logarithmically. Consequently, the amount of information that can be reliably certified, and therefore used as the basis for the mechanism evaluation, can be arbitrarily small relative to the true MI. 

\begin{restatable}[Largest Reliable Lower-Bound for Distribution-Free Estimators]{theorem}{wclb}\label{thm:wclb}
Without prior knowledge of the response distribution, any estimator faces fundamental limits. Let $B$ be any distribution-free estimator providing a $(1-\delta)$ confidence lower bound on $I_f(X;Y)$ (Def. \ref{regMI}), derived from a finite sample empirical type $\mathcal{T}(S^{(N)})$ where $S^{(N)} \sim P_{XY}^{(N)}$. For integers $k \ge 1$ and $N \ge 2$, with probability at least $1-\delta-1/k$ over the sampling:
$$
B\bigl(\mathcal{T}(S^{(N)}), \delta\bigr) \le I_{\text{max}}(2kN^2) := \frac{1}{2kN^2}f(2kN^2) + \left(1 - \frac{1}{2kN^2}\right)f(0).
$$
\end{restatable}

\begin{hproof}
Figure \ref{fig:latent-knowledge} (Right) shows the adversarial “mode collapse’’ construction that drives the bound: keep the largest $kN^{2}$ parts of the response distribution unchanged, spread the next $kN^{2}$ likely responses uniformly at height $1/(2kN^{2})$, and drop the rest. We make this precise by a \emph{maximal coupling} between the true law $P$ and the surrogate $\widetilde P$ that (i) identifies the top $kN^{2}$ atoms, (ii) maps the next $kN^{2}$ atoms to the uniform “orange’’ cloud, and (iii) annihilates the remainder.

Because $\widetilde P$ has only $2kN^{2}$ support points, Lemma \ref{uniform_MI} (Maximum MI) implies \(I_f(\widetilde P)\;\le I_{\text{max}}(2kN^2) .\) This is the dashed level in the figure. Under the coupling, each orange atom under $P$ has mass at most $1/(kN^{2})$. A refined birthday bound on collisions within the orange cloud shows a \emph{pure} sample (no orange repeats) occurs with probability at least $1-\tfrac{1}{k}$. On every pure sample, the empirical type $\mathcal{T}(S^{(N)})$ is identical under $P$ and $\widetilde P$, so the estimator’s $(1-\delta)$ guarantee forces
$B(\mathcal{T}(S^{(N)}),\delta)\le I_f(\widetilde P)$.
Therefore
\(
\Pr\!\bigl[B(\mathcal{T}(S^{(N)}),\delta)>\text{ceiling}\bigr]\;\le\;\delta+\tfrac{1}{k},
\)
which rearranges to the claimed bound.
\end{hproof}

This analysis extends \citet{mcallester2020formal} from Shannon information to general $f$-divergences, revealing that robustness depends on the choice of divergence. Showing this generalization required introducing three techniques. (i) An explicit coupling that aligns $P$ with a $2kN^{2}$-support surrogate, yielding type-indistinguishability on pure samples; (ii) a \emph{maximum MI lemma} (Lemma~\ref{uniform_MI}) showing the uniform coupling extremizes $f$-information under support constraints; and (iii) a sharper failure probability of $\delta+\tfrac{1}{k}$ (improving the previous $1.01/k$) via a tight birthday bound within the orange layer.

While Theorem~\ref{thm:wclb} considers worst-case mathematical constructions, real adversaries employ semantically plausible attacks. Our experiments (Section~\ref{sec:findings}) test four such strategies. Each approximates the theoretical mode collapse by reducing natural variation while preserving semantic content supporting that our theoretical limits capture practical vulnerabilities.

\subsection{Implementing Variational Chain-of-Thought}
\label{sec:certification}
Computing exact mutual information for high-dimensional text is intractable. Instead, we use a variational lower bound implemented as a categorical classification prompt which gives a structured chain-of-thought (CoT) reasoning policy for the overseer.

To estimate TVD-MI, we treat the critic as a binary classifier distinguishing ``real'' pairs (responses to the same source) from ``shuffled" pairs (responses to different sources). We estimate the true positive rate (TPR) and true negative rate (TNR) from LLM outputs and estimate TVD-MI as the sum minus one. This implementation is discriminative rather than a plug-in density estimator over text. One can compute a KL diagnostic after projecting responses through the binary critic, e.g. between the Bernoulli distributions induced by the critic's labels under $P^+$ and $P^-$. However, this critic-output KL is a representation-level quantity and is not the same object as the distribution-free KL-MI certification problem in Theorem~\ref{thm:wclb}; the binary projection can hide the worst-case high-cardinality behavior that drives the KL--TVD separation.

\paragraph{TVD-MI via a binary test (total variation distance).}
Figure \ref{fig:latent-knowledge} (left) shows the step-by-step pipeline and we discuss the formal equations here. Let $P^+ := P_{ij}$ denote the joint distribution of paired responses (same source) and $P^- := P_i \otimes P_j$ denote the product of marginals (independent sources). For total variation distance, the overseer’s reasoning map is a decision rule 
$r : \mathcal{T}(S) \to \{\text{real pair}, \text{shuffled pair}\}$. We take $f(t) = |t-1|$ and this yields
\begin{equation}
I_{\mathrm{TVD}}(Y_i;Y_j) = \mathrm{TVD}(P^+, P^-) \;\geq\; 
\hat I_{\mathrm{TVD}}^r(\mathcal{T}(S)) := \mathrm{TPR}_r + \mathrm{TNR}_r - 1,
\end{equation}
where now
\begin{align}
\mathrm{TPR}_r := \Pr_{S \sim (P^+)^N}[r(\mathcal{T}(S))  = \text{real pair}], \quad
\mathrm{TNR}_r := \Pr_{S \sim (P^-)^N}[r(\mathcal{T}(S)) \neq  \text{real pair}].
\end{align}
The bound is tight when $r$ perfectly separates the distributions \citep[Definition~2.4]{tsybakov2008nonparametric}. This is an instance of Youden's \citeyearpar{youden1950index} J statistic ($\text{TPR} + \text{TNR} - 1$), which measures informativeness \citep{powers2012problem}.

\paragraph{The Decision Rule as Chain-of-Thought Strategy.}
In our implementation, the decision-rule $r$ is realized by an LLM prompt that asks whether two reports appear to come from the same underlying source. This allows the model to use chain-of-thought to determine the final label, but the TVD-MI statistic uses only the final discrete decision (real vs. shuffled pair) when estimating $\mathrm{TPR}_r$ and $\mathrm{TNR}_r$.

\paragraph{TVD-MI as a Principled LLM Judge.}
Our implementation reveals that TVD-MI can be viewed as an LLM judge with different design choices. In terms of prompt structure, we use information relationships ("same source?") vs quality ("which is better?"). In terms of aggregation, we use information-theoretic (TPR + TNR - 1) vs win-rate averaging. Finally, these mechanisms inherit DPI-based gaming resistance vs none for quality-based judging. Both use identical computational resources (single LLM calls), but our information-theoretic framing provides provable robustness properties.

\section{Study Setup}
\label{sec:study-setup}

We designed and pre-registered\footnote{Pre-registration: \url{https://osf.io/c7pum}.} \footnote{Code: \url{https://github.com/zrobertson466920/llm-peer-prediction/tree/main}.} an evaluation study to test whether information-theoretic mechanisms can reliably detect strategic manipulation in AI-generated content. We address three primary research questions, mapped to our pre-registered hypotheses:

\noindent\textbf{RQ1: Can mechanisms detect agent manipulation strategies?} We use Cohen's $d$ (standardized mean difference) between Good Faith and Problematic agents to measure effect-size. We test H1a ($d > 0.5$, medium effect size), H1b (compression effects), H1c (TVD-MI superiority).

\noindent\textbf{RQ2: Do mechanisms produce reliable item-level detection?} We calculate item-level AUC (area under ROC curve) for Faithful--Faithful vs.\ Faithful--Problematic pairs to measure discrimination ability. We test H2c (gaming resistance). We note this was added during analysis as complementary test of pre-registered hypothesis.

\noindent\textbf{RQ3: Do information-theoretic mechanisms resist adversarial attacks?} We measure how performance degrades under four tampering strategies. This tests H2a (bounded consistency), H2b (log-prob degradation), H2c (gaming resistance).

\textbf{Key deviations from pre-registration}: (1) Expanded from 3 to 10 domains for proper compression analysis, (2) Collapsed 4 categories to 2 (Good Faith/Problematic) following our theoretical framework, (3) Added item detection analysis recognizing it directly tests gaming resistance. See Appendix \ref{appendix:methodology} for complete details.

\subsection{Experimental Design}

\textbf{Domain Selection.} We selected 10 domains spanning compression (input/output ratios) from 1.1:1 (translation) to 20.2:1 (peer review). See Appendix \ref{appendix:data} for more details. This range tests mechanisms from near-isomorphic tasks to extreme compression where most information is discarded.

\textbf{Agent Taxonomy.} We developed a taxonomy of 29-30 agent strategies per domain, grouped into two categories: \textbf{good faith} agents that preserve information and \textbf{problematic agents} that degrade it. Good faith agents consist of faithful agents (4-6 variants) that are prompted to accurately complete the task and style agents (10-16 variants) that are prompted to preserve information in an alternative presentation. Problematic agents consist of strategic agents (4-10 variants) that are prompted to deliberately manipulate their completions and low effort agents (4-5 variants) that are prompted to give minimal effort or generic responses. Good Faith agents (Faithful + Style) preserve information while Problematic agents (Strategic + Low Effort) degrade it. Full taxonomy details appear in Appendix \ref{appendix:peerReview} - \ref{appendix:translationDetails}.

\subsection{Evaluation Mechanisms}

We implement three mechanisms with similar per-comparison costs (single API calls), but with different \emph{assumptions}:
(i) MI (DoE) requires token log-probabilities (provider-dependent),
(ii) Judge baselines require either source references or multi-pair aggregation to be reliable in our setting,
and (iii) TVD-MI uses a single black-box categorical decision (same-source vs different-source) and therefore applies to any LLM API.

\textbf{Information-Theoretic Mechanisms:} We evaluate three approaches. The first estimates mutual information via the difference of entropies \textbf{MI (DoE)} using Llama 3.3-70B log probabilities. The second, \textbf{GPPM}, is the generative peer prediction baseline \citep{lu2024eliciting}. The third, \textbf{TVD-MI}, computes mutual information through total variation distance using a categorical critic (GPT-4o-mini).

\textbf{Comparison Methods:} For baselines, we include \textbf{LLM Judge} with and without references \textbf{(w/o ref)}, which uses GPT-4o-mini to assess normative quality via pairwise comparisons. This uses the same model as our TVD-MI critic but prompts for pairwise quality judgments rather than information relationships, with prompt structure following \cite{zheng2023judging}. We also report two standard reference-based metrics: \textbf{ROUGE}, for summarization quality \citep{lin2004rouge}, and \textbf{BLEU}, for translation quality \citep{papineni2002bleu}.

\textbf{Score Aggregation:} For multiple agents, we aggregate pairwise terms. However, when we report AUC we use raw pairwise scores for information-theoretic mechanisms and only aggregate judge scores. This is because individual judge scores are binary preferences, not information scores, so we must aggregate to obtain a win-rate signal. 

\subsubsection{Reducing Potential Confounding}

We reduce potential confounding in two ways. First, in the peer-review setting we employ a fixed ICLR-style template, which constrains stylistic variation. Second, we apply structural normalizations that preserve semantic content while altering statistical properties; these also form the basis of our adversarial robustness experiments. \textbf{Case flipping} alternates character case every fifth position; \textbf{format standardization} removes border markup and inserts context-dependent tags (6-character hashes); \textbf{constant padding} appends a fixed sequence of \texttt{X} characters; and \textbf{pattern injection} inserts context-derived markers (3-character hash prefixes) at regular intervals. These transformations preserve semantics while introducing systematic surface patterns; examples appear in Figure~\ref{fig:latent-knowledge}.

\subsection{Statistical Analyses}

For RQ1 (manipulation detection), we compute paired Cohen's d between Good Faith and Problematic categories with bootstrap CIs. For RQ2 (decomposability), we analyze item-level area under the curve (AUC) detection distinguishing Faithful-Faithful from Faithful-Problematic pairs. For RQ3 (robustness), we apply four adversarial transformations and measure degradation in both d and AUC.

\section{Findings}
\label{sec:findings}

We present empirical validation of our theoretical framework across ten text generation domains. Our results demonstrate that information-theoretic mechanisms with formal guarantees provide substantially more effective detection of strategic manipulation than current evaluation practices.

\subsection{Information-Theoretic Mechanisms Detect Effectively}

All three information theoretic mechanisms successfully discriminate between information preserving and information degrading agents across every tested domain. This supports the theoretical prediction that mechanisms designed for gaming resistance can identify quality differences when quality is aligned with information preservation. Table \ref{tab:effect_sizes} shows how well mechanisms discriminate comparing Good Faith agents (Faithful and Style categories) against Problematic agents (Strategic and Low Effort categories).

For \textbf{information-theoretic mechanisms} designed for strategic robustness, all ten domains achieve $d > 0.5$ across the three mechanisms. The mean effect sizes are substantial with MI (1.87), GPPM (2.70), and TVD-MI (5.20). Mechanisms performed consistently across different compression ratios. In contrast, \textbf{direct quality assessment} methods show weaker results. Using LLM Judge without context, only six of ten domains surpass $d > 0.5$, while with context, nine of ten domains do. Baseline metrics (ROUGE and BLEU) reach this threshold in only six of ten domains.

\begin{table*}[t]
\centering
\caption{Effect sizes (Cohen's $d$) for discrimination between Good Faith and Problematic agents. Cells show mean ± half-width of 95\% confidence interval (CI). (ns) = CI overlaps zero. Bold = $p<0.001$, regular = $p<0.05$, gray = non-significant. P-values adjusted for multiple hypothesis testing.}
\label{tab:effect_sizes}
\footnotesize
\begin{tabular}{@{}lcccccc@{}}
\toprule
\textbf{Domain (Compression)} & \textbf{Baseline} & \textbf{MI (DoE)} & \textbf{GPPM} & \textbf{TVD-MI} & \textbf{Judge (ref)} & \textbf{Judge (no ref)} \\
\midrule
\multicolumn{7}{l}{\textit{Translation}} \\
WMT14 (1.1:1) & 0.93 & \textbf{1.61 ± 0.17} & \textbf{0.70 ± 0.09} & \textbf{3.32 ± 0.27} & \textbf{2.53 ± 0.49} & 0.24 ± 0.20 \\
Opus Books (1.3:1) & \textcolor{gray}{1.22} & \textbf{2.66 ± 0.29} & \textbf{0.73 ± 0.12} & \textbf{3.08 ± 0.34} & \textbf{3.50 ± 0.47} & \textbf{-0.62 ± 0.17} \\
\midrule
\multicolumn{7}{l}{\textit{Summarization}} \\
SamSum (4.8:1) & \textcolor{gray}{0.11} & \textbf{2.52 ± 0.29} & \textbf{2.52 ± 0.27} & \textbf{6.14 ± 0.70} & \textbf{2.70 ± 0.31} & \textbf{0.54 ± 0.15} \\
PubMed (6.7:1) & 0.86 & \textbf{2.01 ± 0.37} & \textbf{3.18 ± 0.57} & \textbf{6.53 ± 0.80} & \textbf{8.14 ± 1.03} & \textbf{3.25 ± 0.54} \\
Multi-News (9.0:1) & 0.88 & \textbf{1.53 ± 0.21} & \textbf{2.70 ± 0.35} & \textbf{6.55 ± 0.96} & \textbf{4.06 ± 0.68} & \textbf{0.54 ± 0.16} \\
BillSum (9.3:1) & 0.91 & \textbf{2.24 ± 0.28} & \textbf{3.59 ± 0.43} & \textbf{5.91 ± 0.82} & \textbf{4.23 ± 0.52} & 0.16 ± 0.14 \\
CNN/Daily (13.8:1) & \textcolor{gray}{0.61} & \textbf{2.06 ± 0.23} & \textbf{3.42 ± 0.40} & \textbf{5.87 ± 0.82} & \textbf{3.55 ± 0.40} & \textbf{0.72 ± 0.11} \\
Reddit TIFU (16.1:1) & \textcolor{gray}{0.13} & \textbf{2.52 ± 0.29} & \textbf{3.76 ± 0.41} & \textbf{7.23 ± 0.94} & \textbf{2.70 ± 0.40} & \textcolor{gray}{0.05 ± 0.14 (ns)} \\
XSum (18.5:1) & \textcolor{gray}{0.29} & \textbf{1.89 ± 0.21} & \textbf{2.85 ± 0.28} & \textbf{6.69 ± 0.77} & \textbf{3.39 ± 0.43} & \textbf{-0.28 ± 0.15} \\
\midrule
\multicolumn{7}{l}{\textit{Peer Review}} \\
ICLR (20.2:1) & \textcolor{gray}{-0.12} & \textbf{0.68 ± 0.24} & \textbf{0.73 ± 0.23} & \textbf{1.82 ± 0.43} & 0.26 ± 0.22 & \textbf{-1.69 ± 0.32} \\
\midrule
\textbf{Success (d > 0.5)} & 6/10 & 10/10 & 10/10 & 10/10 & 9/10 & 4/10 \\
\bottomrule
\end{tabular}
\end{table*}

Using the LLM to implement the TVD-MI critic achieved higher effect sizes when prompting for information relationships than using it to judge normative preferences. TVD-MI has large effect sizes $d = 7$ in several summarization tasks. In Section \ref{sec:decomposability} we also measure AUC to support this finding. Overall, TVD-MI uses the same LLM (GPT-4o-mini) as the quality judge baseline. This suggests that the choice to measure information relationships rather than directly evaluate quality is more important than the sophistication of the implementation. 

Although our initial hypothesis predicted linear degradation with compression, the empirical results instead exhibit an inverted-U pattern (Appendix~\ref{appendix:inverted-u}); this reflects typical-case behavior of specific models and datasets and is not a contradiction of Theorem~\ref{thm:wclb}, whose worst-case sample-complexity bound motivates using bounded $f$-divergences but does not prescribe empirical scaling.

\textbf{Model Ablation.} Our primary experiments use GPT-4o-mini for agent generation, TVD-MI critic, and judge baselines due to its scalability and cost-effectiveness for our ~870 comparisons per example. We replicated a subset of experiments using Llama-3.3-70B-Instruct-Turbo on subsets of Opus Books (n=30) and PubMed (n=100). We also include an embedding baseline that uses OpenAI text-embedding-3-large to score using cosine similarity.

\begin{table}[h]
\centering
\caption{Model ablation reporting Cohen's d using Llama-3.3-70B to implement TVD-MI and Judges.}
\label{tab:model_ablation}
\small
\begin{tabular}{@{}lcccc@{}}
\toprule
\textbf{Domain} & \textbf{Embedding} & \textbf{TVD-MI} & \textbf{Judge (ref)} & \textbf{Judge (no ref)} \\
\midrule
Opus Books (n=30) & \textbf{2.23 ± 1.32} & \textbf{1.85 ± 1.11} & \textbf{4.74 ± 1.62} & 0.41 ± 0.40 \\
PubMed (n=100) & \textbf{3.28 ± 0.40} & \textbf{6.71 ± 1.28} & \textbf{9.68 ± 1.72} & \textbf{4.81 ± 1.06} \\
\bottomrule
\end{tabular}
\end{table}

\subsection{Mechanisms Transform Pairwise Evaluations into Item-Level Detection}
\label{sec:decomposability}

In the previous section we saw that mechanisms achieved large effect-sizes between the good faith and problematic conditions. However, this could be an artifact, and we are interested in measuring the ability to aggregate pairwise comparisons into meaningful item-level detection. We support this finding by showing the large effect-sizes are not artifacts. We do this empirically by testing whether mechanism scores can distinguish faithful--faithful from faithful--problematic pairs without ground truth.

\textbf{Methodology.} For each response item, we classify agent pairs. A \textbf{positive class} consisting of faithful-faithful pairs where both agents preserve information and a \textbf{negative class} of faithful-problematic pairs. For the mechanisms we compute symmetric pairwise scores (averaging directional evaluations) and test whether positive pairs score higher than negative pairs. For the judging baselines this performs poorly so we add an additional multi-pair aggregation step which produces an average win-rate of the response against other responses. We report per-item AUCs macro-averaged across examples with 95\% bootstrap CIs using 10k samples. See more details and results in Appendix \ref{aucDetails}. 

\begin{table}[t]
\centering
\caption{Area under curve (AUC) scores for Faithful-Faithful from Faithful-Problematic agent pairs. Values show macro-averaged AUC ± 95\% CI max-width. Classes are balanced; AUC < 0.5 indicates inversion.}
\label{tab:auc_results}
\begin{tabular}{lccccc}
\toprule
\textbf{Domain} & \textbf{n} & \textbf{MI (DoE)} & \textbf{GPPM} & \textbf{TVD-MI} & \textbf{Judge w/ context} \\
\midrule
\multicolumn{6}{l}{\textit{Translation}} \\
WMT14 & 500 & 0.664 ± 0.006 & 0.703 ± 0.006 & \textbf{0.710 ± 0.006} & 0.654 ± 0.011 \\
OPUS & 186 & 0.737 ± 0.010 & \textbf{0.743 ± 0.009} & 0.703 ± 0.008 & \textbf{0.743 ± 0.010}  \\
\midrule
\multicolumn{6}{l}{\textit{Summarization}} \\
BillSum & 200 & 0.692 ± 0.008 & 0.677 ± 0.007 & \textbf{0.732 ± 0.007} & 0.675 ± 0.008  \\
CNN/DM & 268 & 0.706 ± 0.007 & 0.669 ± 0.006 & \textbf{0.762 ± 0.005} & 0.686 ± 0.009  \\
MultiNews & 200 & 0.695 ± 0.010 & 0.674 ± 0.008 & \textbf{0.755 ± 0.007} & 0.726 ± 0.009  \\
PubMed & 200 & 0.700 ± 0.008 & 0.698 ± 0.007 & \textbf{0.753 ± 0.007} & 0.717 ± 0.007 \\
Reddit TIFU & 200 & 0.689 ± 0.008 & 0.638 ± 0.008 & \textbf{0.772 ± 0.008} & 0.655 ± 0.011  \\
SAMSum & 200 & 0.655 ± 0.008 & 0.645 ± 0.007 & \textbf{0.754 ± 0.007} & 0.655 ± 0.010  \\
XSum & 200 & 0.714 ± 0.008 & 0.694 ± 0.007 & \textbf{0.767 ± 0.006} & 0.645 ± 0.010 \\
\midrule
\multicolumn{6}{l}{\textit{Peer Review}} \\
ICLR & 100 & 0.484 ± 0.008 & 0.417 ± 0.010 & \textbf{0.544 ± 0.007} & 0.483 ± 0.007 \\
\bottomrule
\end{tabular}
\end{table}

\textbf{Results.} Table~\ref{tab:auc_results} shows TVD-MI achieves the strongest discrimination across nearly all domains (0.71-0.77 for translation/summarization) with a mean AUC of 0.73 while the judge has a mean of 0.66. Unlike TVD-MI scoring, the judge requires a reference and aggregated scoring. The peer review domain proves challenging for all methods due to extreme compression (20:1), though TVD-MI remains above random. Because classes are balanced, AUC = 0.5 corresponds to random guessing; values below 0.5 indicate \textit{systematic inversion} (faithful–problematic pairs ranked above faithful–faithful), not noise.

\subsection{Gaming-Resistance: Information-Theoretic Mechanisms Show Superior Robustness}

Theorem \ref{thm:wclb} indicates that transformations inducing tighter or ``mode-collapsed'' uniformity can make some mutual-information estimators statistically difficult to certify in the worst case. Our adversarial attacks (Figure \ref{fig:latent-knowledge}) are inspired by this construction: for example, deterministic case-flipping creates an artificial mode in the response distribution. All four transformations (case flipping, format standardization, constant padding, pattern injection) preserve broad task identity while altering surface form. These attacks are not exhaustive adaptive attacks; rather, they test whether robustness predicted by the theorem appears under semantically plausible shared-pattern manipulations. Experimental results support this comparative claim: information-theoretic mechanisms see degraded discrimination under attacks, but TVD-MI maintains better AUC robustness than the LLM judge baselines. Table \ref{tab:tampering_comprehensive} presents the effects of four adversarial transformations on mechanism performance. We also report AUC in Table \ref{tab:auc_drop}.

\begin{table*}[t]
\centering
\caption{Effects of adversarial transformations on mechanism scores and effect-size for Reddit TIFU. Changes show mean difference $\pm$ 95\% CI max width. Effect-size degradation shows change in Cohen's d. Bold indicates p < 0.001, regular text p < 0.05, gray text non-significant. Red values indicate degradation ($\Delta$d < -0.3).}
\label{tab:tampering_comprehensive}
\footnotesize
\begin{tabular}{@{}lccccc@{}}
\toprule
\textbf{Transformation} & \textbf{MI} & \textbf{GPPM} & \textbf{TVD-MI} & \textbf{Judge} & \textbf{Judge} \\
& \textbf{(DoE / GEM)} & & & \textbf{(w/ ref)} & \textbf{(w/o ref)} \\
\midrule
\multicolumn{6}{l}{\textit{Score Changes ($\Delta$)}} \\
Case Flip & \textbf{-0.032$\pm$0.050} & \textcolor{gray}{-0.014$\pm$0.050} & \textbf{+0.070$\pm$0.050} & \textbf{-0.111$\pm$0.050} & \textbf{-0.110$\pm$0.050} \\
Format & \textbf{+0.455$\pm$0.050} & \textbf{+0.233$\pm$0.050} & \textbf{+0.077$\pm$0.050} & \textcolor{gray}{+0.000$\pm$0.050} & \textbf{-0.042$\pm$0.050} \\
Padding & \textbf{+0.201$\pm$0.050} & \textbf{+0.080$\pm$0.050} & \textbf{+0.029$\pm$0.050} & \textbf{-0.064$\pm$0.050} & \textbf{-0.101$\pm$0.050} \\
Pattern & \textbf{+0.214$\pm$0.050} & \textbf{+0.965$\pm$0.050} & \textbf{+0.113$\pm$0.050} & \textbf{-0.338$\pm$0.050} & \textbf{-0.479$\pm$0.050} \\
\textbf{Average} & +0.209$\pm$0.172 & +0.316$\pm$0.385 & +0.072$\pm$0.030 & -0.128$\pm$0.127 & -0.183$\pm$0.173 \\
\midrule
\multicolumn{6}{l}{\textit{Discrimination Degradation ($\Delta$ Cohen's d)}} \\
Case Flip & \textcolor{red}{-1.252} & \textcolor{red}{-0.540} & \textcolor{red}{-2.259} & \textcolor{red}{-1.090} & \textcolor{red}{-1.000} \\
Format & \textcolor{red}{-2.106} & +0.138 & \textcolor{red}{-1.336} & +0.096 & +0.273 \\
Padding & \textcolor{red}{-1.413} & -0.238 & \textcolor{red}{-0.438} & -0.015 & +0.364 \\
Pattern & \textcolor{red}{-3.441} & -0.115 & \textcolor{red}{-1.900} & \textcolor{red}{-2.074} & -0.046 \\
\textbf{Average} & \textcolor{red}{-2.053} & -0.189 & \textcolor{red}{-1.483} & \textcolor{red}{-0.771} & -0.102 \\
\bottomrule
\end{tabular}
\end{table*}

\begin{table*}[t]
\centering
\caption{Effects of adversarial transformations on mechanism discrimination ability (AUC) for Reddit TIFU summarization. Bold indicates highest score.}
\label{tab:auc_drop}
\begin{tabular}{lccccc}
\toprule
\textbf{Attack} & \textbf{MI} & \textbf{GPPM} & \textbf{TVD-MI} & \textbf{Judge w/} & \textbf{Judge w/o} \\
\midrule
Case Flip & 0.618 & 0.562 & \textbf{0.707} & 0.598 & 0.456 \\
Format & 0.582 & 0.575 & \textbf{0.745} & 0.667 & 0.498 \\
Padding & 0.614 & 0.602 & \textbf{0.759} & 0.649 & 0.505 \\
Pattern & 0.552 & 0.606 & \textbf{0.714} & 0.536 & 0.500 \\
\bottomrule
\end{tabular}
\end{table*}

\textbf{Gaming resistance reveals paradoxical patterns.} 
TVD-MI scores increase consistently under all attacks (+0.029 to +0.113 in raw score), yet it remains strongly discriminative on average ($d = 7.24;\ \bar \Delta d=-1.483$). 
This is consistent with our theoretical prediction: linear-growth $f$-MI prevents score deflation but generally cannot prevent adversaries from adding spurious patterns that obscure meaningful distinctions. 
In contrast, super-linear MI shows higher vulnerability, with an average score inflation of +0.209 coupled with a large discrimination drop ($d =3.76;\ \bar \Delta d=-2.053$).

\textbf{Theory correctly predicts relative robustness hierarchies.} 
Theorem~\ref{thm:wclb} predicts linear-growth measures should maintain better guarantees than super-linear ones under adversarial conditions. The results support this: TVD-MI averages $\bar \Delta d=-1.483$, MI/DoE $\bar \Delta d=-2.053$, while GPPM shows relatively small change ($\bar \Delta d=-0.189$). LLM judges exhibit variable behavior, ranging from large drops to spurious gains. TVD-MI maintains AUC > 0.70 under all attacks, while judges degrade to random performance (near 0.50) and other mechanisms show larger degradation.

\textbf{Notable vulnerability of quality-based prompting.} 
The same LLM (GPT-4o-mini) prompted for quality judgments shows failure under case-flipping ($d:0.05\to-0.95;\ \Delta d=-1.000$) and near-complete inversion under pattern injection ($d:0.05\to0.00;\ \Delta d=-0.046$). Under padding it spuriously improves ($d:0.05\to0.41;\ \Delta d=+0.364$). These shifts indicate the judge has lost meaningful connection to content quality, reacting instead to surface features. The consistent pattern across transformations demonstrates that adversarial robustness is a distinct challenge from score manipulation. While we cannot prevent all gaming, the robustness gap provides a clear design principle for practical deployment. 

\section{Discussion}

Our findings reveal a fundamental insight: the same LLM that fails as a quality judge succeeds as an information detector. This reframing from normative quality judgments to information relationships provides both theoretical and practical advantages. This shift reflects an important insight about AI evaluation: when ground truth is unavailable, measuring what agents preserve (information) proves more robust than measuring what they produce (quality).

LLM judges often invert rankings (Table~\ref{tab:effect_sizes}) and show low robustness under attack (Table~\ref{tab:tampering_comprehensive}). Appendix Table~\ref{tab:llm_failures} provides a concrete example from CNN/DailyMail in which fabricated ``Conspiracy Theory'' and another from ICLR in which ``Method Shift'' strategies receive higher judge scores than human reference summaries. In contrast, the same models reliably detect information relationships, providing stable discrimination even under manipulation (Table~\ref{tab:auc_drop}). This suggests that LLMs are well suited to detecting statistical structure but less reliable at implicit value judgments.

Our mechanism also supports item-level detection. TVD-MI achieves AUC 0.70--0.77 in distinguishing faithful--faithful from faithful--problematic pairs in most domains (Table~\ref{tab:auc_results}), outperforming LLM judges in nine of ten domains despite the judge having access to reference context. This decomposition succeeds because TVD-MI relates to classical reliability measures (Cohen’s $\kappa$, Youden’s $J$) that capture absolute agreement rather than relative preference.

\subsection{Limitations and Future Directions}

\textbf{Adversarial Robustness.} While TVD-MI remains above 0.70 AUC under attack, adversaries reduce performance from 0.77 to 0.71, indicating room for adaptive defenses. Our experiments do not test fully adaptive overlap attacks in which an agent explicitly optimizes to maximize agreement with another response to the same source, for example by adding as much source-specific detail as possible regardless of usefulness. Such attacks are outside the modeled post-processing class and could increase measured overlap while degrading human-perceived quality. Boundedness of TVD caps raw score inflation, but it does not by itself prevent ranking manipulation; addressing this class of attacks likely requires task-specific constraints, richer critics, or additional quality controls.

\textbf{Extreme compression.} Peer review (20:1 compression) remains difficult for all methods. When most information is discarded, distinguishing preservation strategies becomes inherently limited; domain-specific calibration may help.

\textbf{Dependence on pre-trained knowledge.} Our mechanisms rely on LLM priors, which may fail in unfamiliar domains. The gap between empirical performance and our worst-case bounds (Theorem~\ref{thm:wclb}) highlights the importance of exploring learned overseers or reinforcement learning approaches.

\section{Conclusions}

Information-based evaluation provides what normative quality judgments often cannot: robustness and clear theoretical grounding. Theorem~\ref{thm:wclb} shows that bounded $f$-divergences maintain polynomial sample complexity under adversarial manipulation, motivating our use of TVD-MI. We implement this via a black-box, binary ``same source?'' critic and demonstrate effectiveness across ten domains: detecting strategic manipulation, producing item-level detection, and maintaining performance under attacks that reduce LLM judges to random.

As AI systems increasingly evaluate AI-generated content, mechanisms grounded in information relationships may be necessary to prevent evaluation collapse. Our results show that robust, ground-truth-free assessment is possible with current models, provided we ask them the right questions.

\paragraph{Statement of Broader Impact}

Our findings arrive as organizations increasingly rely on LLM judges for critical decisions, from content moderation to scientific peer review. Information-theoretic mechanisms require no special access, democratizing robust evaluation. However, our mechanism measures information shared with peers, not correctness. Minority or majority viewpoints may not receive lower or higher scores, respectively. We therefore view mutual evaluation as a modeling framework and TVD-MI primarily as an information signal, to be used alongside domain-specific judgment rather than as a standalone ranking mechanism. While revealing these vulnerabilities might accelerate adversarial behavior, the greater risk lies in continued reliance on manipulable judges.

\paragraph{Acknowledgment}

We would like to thank Yuqing Kong, Florian Dorner, Mike Hardy, and Yijia Shao for their feedback to the paper. SK acknowledges support by NSF 2046795 and 2205329, IES R305C240046, ARPA-H, the MacArthur Foundation, Schmidt Sciences, OpenAI, and Stanford HAI.

\newpage

\bibliographystyle{tmlr}
\bibliography{Lets_Measure}

\newpage

\appendix

\section{Extended Study Methodology}
\label{appendix:methodology}

\subsection{Pre-Registration and Analysis Evolution}
\label{appendix:preReg}
Our pre-registered study (blinded for review) originally focused on paired Cohen's d effect sizes to test discrimination between agent categories. The pre-registration specified:

\textbf{H1: Information Preservation Detection}
\begin{itemize}
\item H1a: All mechanisms distinguish Problematic from Good Faith agents (d > 0.5)
\item H1b: Detection ability decreases linearly with compression ratio
\item H1c: TVD-MI shows most robust detection across compression levels
\end{itemize}

\textbf{H2: Mechanism Properties}
\begin{itemize}
\item H2a: Bounded mechanisms (TVD-MI) show more consistent performance
\item H2b: Log-probability mechanisms degrade in high-compression domains
\item H2c: Gaming resistance highest for TVD-MI (tested via tampering experiments)
\end{itemize}

During our pre-registration dialogue with an independent AI reviewer (included in the OSF registration), we recognized that validating our decomposability assumption, a fixed oversight strategy across pairs would be effective, required item-level analysis beyond aggregate effect sizes. This led us to implement AUC analysis examining whether item-wise scores could distinguish agent quality levels. Specifically, we test whether scores for Faithful-Faithful pairs exceed scores for Faithful-Problematic pairs at the item level, providing both a validation of decomposability and a complementary test of H2c (gaming resistance) beyond our planned tampering experiments.

\subsection{Complete Agent Taxonomy}

Our agent taxonomy was designed to test different forms of information preservation and degradation. Each category serves a specific purpose:

\textbf{Good Faith Agents (Information-Preserving):}
\begin{itemize}
\item \textbf{Faithful}: Strategies that prioritize accurate information transfer without stylistic modifications. These serve as our primary positive examples.
\item \textbf{Style}: Strategies that alter presentation (tone, register, framing) while attempting to preserve semantic content. These test whether mechanisms can distinguish style from substance.
\end{itemize}

\textbf{Problematic Agents (Information-Degrading):}
\begin{itemize}
\item \textbf{Strategic}: Strategies that deliberately manipulate, misrepresent, or distort information. These test detection of adversarial behavior.
\item \textbf{Low Effort}: Strategies that provide minimal information through laziness, over-compression, or generic responses. These test detection of low-quality outputs.
\end{itemize}

The complete taxonomy for each domain appears in Tables \ref{tab:translation_conditions}, \ref{tab:summarization_conditions}, and \ref{tab:peer_review_conditions}.

\textbf{Category Evolution:} Our pre-registration initially considered four separate categories. During exploratory analysis, we recognized that the basic distinction was between information-preserving (Good Faith: Faithful + Style) and information-degrading (Problematic: Strategic + Low Effort) behaviors, leading to our two-category framework. Both analyses are reported for transparency.

\subsection{AUC Computation Methodology}
\label{aucDetails}
For each source item, we compute mechanism scores for all agent pair combinations. The AUC analysis proceeds as follows:

\begin{enumerate}
\item \textbf{Pair Classification}:
   \begin{itemize}
   \item Positive class: Faithful-Faithful pairs (both agents from Faithful category)
   \item Negative class: Faithful-Problematic pairs (one Faithful, one Strategic/Low Effort)
   \end{itemize}

\item \textbf{Score Computation}:
   \begin{itemize}
   \item MI/GPPM: Symmetrize by averaging $(A,B)$ and $(B,A)$ directions
   \item TVD-MI: Use bidirectional critic score
   \item Judge: Convert pairwise preferences to relative quality scores (winner=1, loser=0) using the average win-rate against all other conditions on that item. We take the average of the column vs. row rates.
   \end{itemize}

\item \textbf{Statistical Analysis}:
   \begin{itemize}
   \item Compute per-item AUC (rank positive pairs above negative pairs)
   \item Report macro-average across items to avoid pooling bias
   \item Bootstrap 95\% CIs by resampling items (10k iterations)
   \end{itemize}
\end{enumerate}

\subsection{Experimental Configurations}
\label{appendix:data}
All experiments used consistent configurations across domains to minimize confounds:

\textbf{Datasets:}
We intentionally select benchmarks that are open-ended across translation, summarization, and peer-review.

\textbf{Translation.}
WMT14 news translation shared task; we use a 500-example subset \citep{bojar-EtAl:2014:W14-33} and OPUS \citep{tiedemann-2016-opus}.

\textbf{Summarization.}
BillSum \citep{kornilova-eidelman-2019-billsum}, CNN/DailyMail \citep{hermann2015teaching,see-etal-2017-get}, MultiNews \citep{fabbri-etal-2019-multi}, PubMed \citep{cohan2018discourse}, Reddit TIFU \citep{kim-etal-2019-abstractive}, SAMSum \citep{gliwa-etal-2019-samsum}, and XSum \citep{narayan-etal-2018-dont}.

\textbf{Peer-Review.}
ICLR reviews collected via OpenReview \citep{openreview}; see also the PeerRead corpus \citep{kang18naacl}.

\textbf{Agent Response Generation:}
\begin{itemize}
\item Model: GPT-4o-mini
\item Temperature: 0.7
\item Max tokens: 150 (summarization), 2000 (peer review), unbounded (translation)
\item Identical base prompts with condition-specific modifications
\end{itemize}

\textbf{Mechanism Evaluation:}
\begin{itemize}
\item MI/GPPM: Llama 3.3-70B-Instruct for log probabilities
\item TVD-MI: GPT-4o-mini for categorical critic
\item Judge: GPT-4o-mini for pairwise comparison
\item All evaluations at temperature 0.0 for consistency
\end{itemize}

\subsection{Computational Requirements}

Our comprehensive evaluation involved:
\begin{itemize}
\item 10 domains × 100-500 items × 30 conditions = ~135,000 agent responses
\item 870 pairwise comparisons per item = ~4.35 million evaluation calls
\item Approximately 500 million tokens processed
\item 72 hours of API computation time
\item Total cost did not exceed 1.5\$ for original design. 
\item The GPT-4o-mini (OpenAI) model is approximately five times less expensive to run as judge/critic than Llama 3.3-70B-Instruct (Together.AI)
\end{itemize}

Despite this scale, deployment requires only single API calls per evaluation, making our mechanisms practical for real-world use.

\subsection{TVD-MI Critic and Judge Prompts}

Agent responses were evaluated using two primary mechanisms:

\subsubsection{TVD-MI Critic}
The TVD-MI critic assessed whether two reviews showed evidence of addressing the same source paper:

\begin{verbatim}
You are evaluating whether two responses could be from agents working 
on the same task.

Task description: Scientific paper peer review task

Response A: [Review A]
Response B: [Review B]

Do these responses show evidence of coming from the same task/source? 
Consider:
- Shared specific details, facts, or entities
- Similar topics or themes
- Overlapping information that would be unlikely if from different sources

Rate the information gain:
- [[Significant Gain]]: Clear evidence they're from the same source
- [[Little Gain]]: Some shared elements but also differences  
- [[No Gain]]: No evidence of shared source
\end{verbatim}

\subsubsection{LLM Judge}
The judge evaluated review quality with and without access to the original reference question. We call this ``ref''  and ``no ref'' in tables. 

\begin{verbatim}
[System]
Please act as an impartial judge and evaluate the quality of the responses provided by two
AI assistants to the user question displayed below. You should choose the assistant that
follows the user's instructions and answers the user's question better. Your evaluation
should consider factors such as the helpfulness, relevance, accuracy, depth, creativity,
and level of detail of their responses. Begin your evaluation by comparing the two
responses and provide a short explanation. Avoid any position biases and ensure that the
order in which the responses were presented does not influence your decision. Do not allow
the length of the responses to influence your evaluation. Do not favor certain names of
the assistants. Be as objective as possible. After providing your explanation, output your
final verdict by strictly following this format: "[[A]]" if assistant A is better, "[[B]]"
if assistant B is better, and "[[C]]" for a tie.

[Task Description]
{task_description}

[User Question]
{query}

[The Start of Assistant A's Answer]
{response_a}
[The End of Assistant A's Answer]

[The Start of Assistant B's Answer]
{response_b}
[The End of Assistant B's Answer]
\end{verbatim}

The prompts are similar up to the presence of the user question section. The no reference condition removes this section. 

\section{Additional Findings}

\subsection{The Inverted-U Pattern: Compression and Information Structure}
\label{appendix:inverted-u}

Contrary to our pre-registered hypothesis of linear degradation, mechanism performance exhibited an inverted-U relationship with both compression ratio and information structure. This pattern reflects a classical bias-variance trade-off: at low compression, agents produce near-identical outputs (high bias, low variance), while at extreme compression, responses become too noisy to distinguish strategies (low bias, high variance). Optimal discrimination occurs at intermediate compression where agent strategies create distinguishable but stable patterns.

For compression ratio, quadratic models significantly outperformed linear fits for all primary mechanisms. GPPM showed the most improvement (R² increasing from 0.029 to 0.684, p = 0.007), while TVD-MI exhibited similar gains (R² from 0.046 to 0.674, p = 0.008). The quadratic coefficient was negative for all mechanisms, confirming the inverted-U shape with peaks at compression ratios of 9.6:1 (MI), 11.0:1 (GPPM), and 11.2:1 (TVD-MI).

\begin{figure}[t]
\centering
\includegraphics[width=\textwidth]{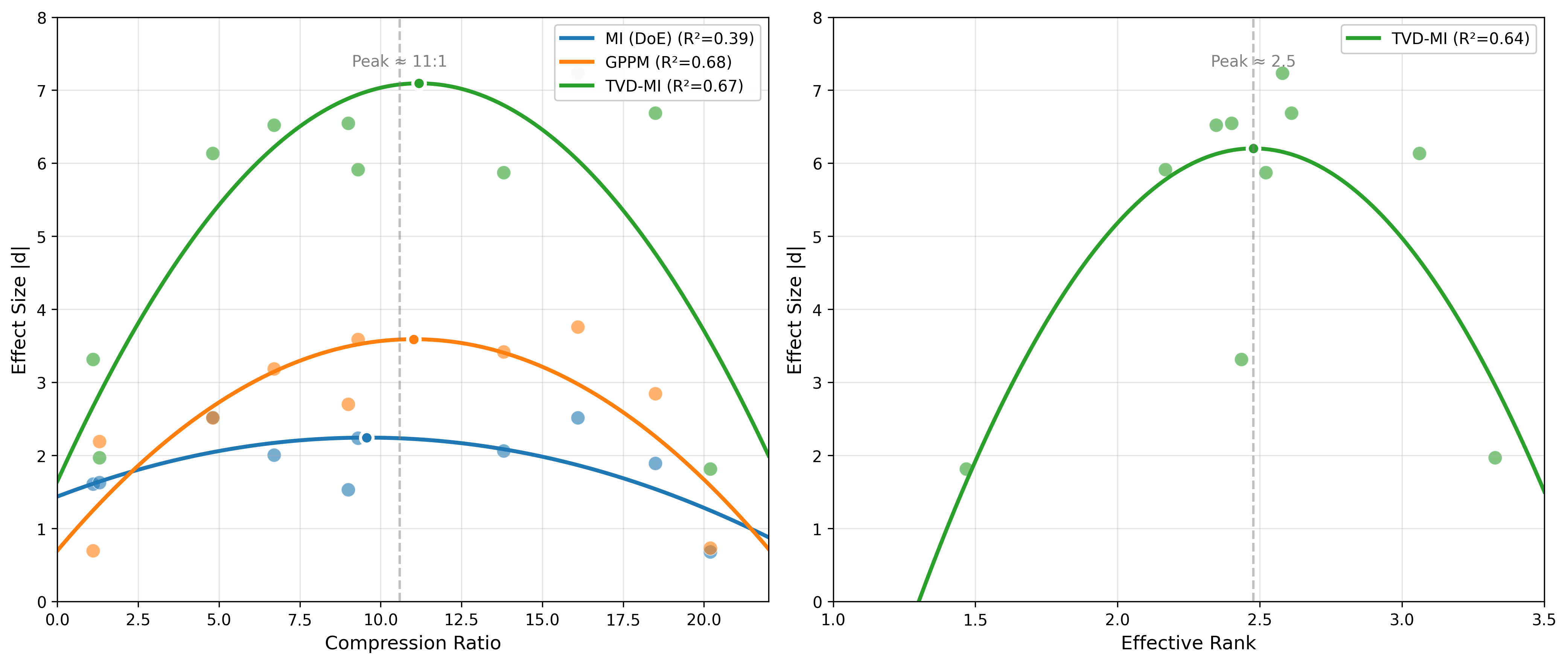}
\caption{Effect sizes for information-theoretic mechanisms exhibit inverted-U relationships with both compression ratio and information structure. \textbf{Left}: Performance peaks at moderate compression ratios (10:1) across all mechanisms. \textbf{Right}: TVD-MI effect size as a function of effective rank, a measure of information diversity in agent response patterns, shows optimal discrimination at approximately 3 effective dimensions. Quadratic models (solid lines) significantly outperform linear fits for both relationships (p < 0.01), revealing that mechanisms achieve peak performance not at extremes but at intermediate levels of information complexity where agent strategies are maximally distinguishable.}
\label{fig:inverted-u}
\end{figure}

The relationship became clearer when we explored the information structure through stable rank, a measure of the dimensional complexity of agent response patterns \citep{recht2010guaranteed}. Figure \ref{fig:inverted-u} presents both relationships. The effective rank analysis yielded the strongest fit (R² = 0.677, p < 0.01), with the quadratic model revealing optimal performance at approximately 3 effective dimensions. This suggests mechanisms work best when agent strategies create distinguishable clusters without excessive noise.

\subsection{LLM Judge Without Reference Can Produce Inverted Evaluations}

While information theoretic mechanisms demonstrated consistent success, the LLM based judge exhibited evaluation inversions beyond simple inaccuracy. In the highest compression domains, the LLM judge without access to context inverted quality rankings, assigning higher scores to problematic content than to good faith responses.

The LLM judge without context showed significant negative effect sizes in two domains: XSum (d = -0.28, p < 0.001) and ICLR peer review (d = -1.69, p < 0.001). This means the judge consistently rated Strategic and Low Effort agents higher than Faithful and Style-preserving agents. Even with access to source material, the judge achieved only marginal discrimination in peer review (d = 0.26). 

Table \ref{tab:llm_failures} illustrates these inversions across two domains. In CNN/Daily Mail summarization, human reference summaries received the lowest scores from LLM judges (0.234 with context, 0.117 without), while ``Conspiracy Theory'' summaries that fabricated information scored 0.703 and 0.777 respectively. The pattern repeats in ICLR peer review: actual ICLR reviewer reports (Reference) score 0.436 and 0.383 from the judges, while ``Method Shift'' reviews, which systematically misrepresent the paper's methodology, score 0.913 without context, exceeding even Faithful reviews (0.909). Reference-based ROUGE-1 fails to distinguish these manipulations under the structured ICLR review template (0.256 for Method Shift vs.\ 0.267 for Faithful), whereas TVD-MI separates them clearly (0.758 vs.\ 0.901).

\begin{table}[h]
\centering
\caption{Evaluation scores by condition type for two domains. Human references represent ground truth quality. Higher scores indicate better perceived quality for all metrics.}
\label{tab:llm_failures}
\begin{tabular}{llcccc}
\toprule
\textbf{Domain} & \textbf{Condition} & \textbf{ROUGE-1} & \textbf{TVD-MI} & \textbf{Judge} & \textbf{Judge} \\
& & & & \textbf{(w/ context)} & \textbf{(w/o context)} \\
\midrule
\multirow{5}{*}{CNN/DM}
 & Human Reference     & 1.000 & 0.636 & 0.234 & 0.117 \\
 & Faithful            & 0.259 & 0.702 & 0.876 & 0.832 \\
 & Fact Manipulation   & 0.194 & 0.371 & 0.324 & 0.672 \\
 & Conspiracy Theory   & 0.193 & 0.623 & 0.703 & 0.777 \\
 & Low Effort          & 0.123 & 0.386 & 0.355 & 0.370 \\
\midrule
\multirow{5}{*}{ICLR}
 & Human Reference           & 1.000 & 0.818 & 0.436 & 0.383 \\
 & Faithful            & 0.267 & 0.901 & 0.698 & 0.909 \\
 & Result Manipulation & 0.263 & 0.893 & 0.565 & 0.750 \\
 & Method Shift        & 0.256 & 0.758 & 0.615 & 0.913 \\
 & Low Effort          & 0.265 & 0.894 & 0.539 & 0.777 \\
\bottomrule
\end{tabular}
\end{table}

\newpage

\section{Peer Review Experiment Details}
\label{appendix:peerReview}

The peer review experiment evaluated 100 ICLR 2023 papers using 30 review strategies designed to test different forms of strategic manipulation and bias in scientific evaluation.

\subsection{Experimental Configuration}
We used the following configuration for the peer review generation:

\begin{verbatim}
PEER_REVIEW_CONFIG = {
    "task_type": "peer_review",
    "task_description": "Scientific paper peer review task",
    "model_config": {
        "model_name": "gpt-4o-mini",
        "max_tokens": 2000,
        "temperature": 0.7
    },
    "data_config": {
        "input_data_path": "data/iclr2023_500.json",
        "sample_size": 100
    }
}
\end{verbatim}

\subsection{Agent Review Strategies}

Table \ref{tab:peer_review_conditions} presents the complete taxonomy of review strategies, organized by their approach to information manipulation and evaluation quality.

\begin{table}[h]
\centering
\caption{Peer review strategy taxonomy with categorization for binary discrimination analysis. Each strategy represents a different approach to reviewing scientific papers.}
\label{tab:peer_review_conditions}
\small
\begin{tabular}{llp{0.45\textwidth}l}
\toprule
\textbf{ID} & \textbf{Strategy} & \textbf{Description} & \textbf{Category} \\
\midrule
1 & Reference & Actual human reviews from ICLR reviewers & Faithful \\
\midrule
\multicolumn{4}{l}{\textit{Faithful Strategies}} \\
2 & Faithful & Detailed, accurate notes capturing actual content & Faithful \\
3 & Objective Analysis & Comprehensive documentation of strengths/limitations & Faithful \\
4 & Thorough Evaluation & Meticulous coverage of all paper aspects & Faithful \\
\midrule
\multicolumn{4}{l}{\textit{Style Strategies}} \\
5 & Balanced Critique & Equal emphasis on positive and negative aspects & Style \\
6 & Overly Technical & Excessive technical language and jargon & Style \\
7 & Harsh Critique & Disproportionate focus on weaknesses & Style \\
8 & Overly Positive & Excessive enthusiasm, glossing over limitations & Style \\
9 & Theory Focus & Overemphasis on theoretical aspects & Style \\
10 & Implementation Obsessed & Excessive focus on implementation details & Style \\
11 & Comparison Fixated & Constant unfair comparisons to other works & Style \\
12 & Pedantic Details & Obsession with minor notation/formatting issues & Style \\
13 & Scope Creep & Suggesting unreasonable scope extensions & Style \\
14 & Statistical Nitpick & Excessive focus on statistical methodology & Style \\
15 & Future Work Focus & Emphasis on limitations rather than contributions & Style \\
16 & Writing Critique & Primary focus on presentation over content & Style \\
\midrule
\multicolumn{4}{l}{\textit{Strategic Strategies}} \\
17 & Method Shift & Misrepresent methodology while using similar terms & Strategic \\
18 & Question Shift & Describe as addressing completely different problem & Strategic \\
19 & Contribution Misrepresent & Systematically misrepresent significance of contributions & Strategic \\
20 & Result Manipulation & Subtly alter reported experimental results & Strategic \\
21 & Assumption Attack & Attack fundamental assumptions unfairly & Strategic \\
22 & Dismissive Expert & Condescending attitude, finding work trivial & Strategic \\
23 & Agenda Push & Push toward reviewer's preferred methods/agenda & Strategic \\
24 & Benchmark Obsessed & Judge primarily on benchmark performance numbers & Strategic \\
\midrule
\multicolumn{4}{l}{\textit{Low Effort Strategies}} \\
25 & Low Effort & Extremely brief notes, minimal detail & Low Effort \\
26 & Generic & Generic notes applicable to any ML paper & Low Effort \\
27 & Surface Skim & Notes based only on abstract/introduction/conclusion & Low Effort \\
28 & Template Fill & Minimal notes following standard template & Low Effort \\
29 & Checklist Review & Standard checklist without deep engagement & Low Effort \\
\bottomrule
\end{tabular}
\end{table}

\subsection{Two-Stage Review Generation Process}

All agents except Reference (which used actual ICLR reviews) followed a two-stage prompting process:

\subsubsection{Stage 1: Reading and Note-Taking}
Each strategy involved specific instructions for reading the paper and taking notes. For example:

\textbf{Strategic Rewrite (Method Shift):}
\begin{verbatim}
Take notes on the paper for an ICLR style review. In your notes, 
make it appear that the paper presents a significantly different 
methodology while maintaining similar terminology and claimed results. 
Just output the notes.
\end{verbatim}

\textbf{Low Effort:}
\begin{verbatim}
Take extremely brief notes on the paper for an ICLR style review. 
Limit yourself to 2-3 sentences per section, avoiding detailed 
explanations of methods or results. Just output these minimal notes.
\end{verbatim}

\subsubsection{Stage 2: Review Generation}
All strategies then used the same template to generate structured ICLR-style reviews:

\begin{verbatim}
Create an ICLR-style review following this specific structure:

# Summary Of The Paper
Summarize the paper's main contributions, methodology, and findings.

# Strength And Weaknesses
Analyze the paper's contributions based on your notes.

# Clarity, Quality, Novelty And Reproducibility
Evaluate based on your notes.

# Summary Of The Review
Provide a 2-3 sentence distillation of your overall assessment.

# Correctness
Rate on a scale of 1-5.

# Technical Novelty And Significance
Rate on a scale of 1-5.

# Empirical Novelty And Significance
Rate on a scale of 1-5.

Maintain a professional tone throughout. Base your review entirely 
on your reading notes.
\end{verbatim}

\section{Summarization Experiment Details}
\label{appendix:summarizationDetails}

The summarization experiments evaluated how well our mechanisms detect information manipulation across seven diverse summarization tasks with compression ratios ranging from 4.8:1 to 18.5:1.

\subsection{Dataset Overview}
We tested on the following summarization datasets, each providing different challenges:
\begin{itemize}
\item \textbf{SamSum (4.8:1)}: Dialogue summarization with lowest compression
\item \textbf{PubMed (6.7:1)}: Medical abstract summarization
\item \textbf{Multi-News (9.0:1)}: Multi-document news summarization
\item \textbf{BillSum (9.3:1)}: Legislative bill summarization
\item \textbf{CNN/Daily Mail (13.8:1)}: Single-document news summarization
\item \textbf{Reddit TIFU (16.1:1)}: Social media story summarization
\item \textbf{XSum (18.5:1)}: Extreme summarization with highest compression
\end{itemize}

For detailed analysis, we focus on CNN/Daily Mail as a representative example. We extracted 500 article-summary pairs using a streaming approach. Articles averaged 619.3 words, while reference summaries averaged 43.3 words, yielding a compression ratio of 13.8:1. This intermediate compression ratio provides a rich testbed for strategic manipulation while maintaining some verifiability against source material.

\subsection{Experimental Configuration}
We used the following configuration for data generation:

\begin{verbatim}
SUMMARIZATION_CONFIG = {
    "task_type": "summarization",
    "task_description": "CNN/DailyMail article summarization task",
    "add_references": true,
    "model_config": {
        "model_name": "gpt-4o-mini",
        "max_tokens": 150,
        "temperature": 0.7
    },
    "data_config": {
        "input_data_path": "data/cnn_dailymail_500.json",
        "sample_size": 500
    }
}
\end{verbatim}

\subsection{Agent Summarization Strategies}

Table \ref{tab:summarization_conditions} presents the complete taxonomy of summarization strategies, organized by their approach to information manipulation and content fidelity.

\begin{table}[h]
\centering
\caption{Summarization strategy taxonomy with categorization for binary discrimination analysis. Each strategy represents a different approach to summarizing news articles.}
\label{tab:summarization_conditions}
\small
\begin{tabular}{llp{0.45\textwidth}l}
\toprule
\textbf{ID} & \textbf{Strategy} & \textbf{Description} & \textbf{Category} \\
\midrule
\multicolumn{4}{l}{\textit{Faithful Strategies}} \\
1 & Faithful & Summarize key points in 3-4 bullet points accurately & Faithful \\
2 & Objective & Focus on factual content without bias & Faithful \\
3 & Comprehensive & Thorough coverage of all important aspects & Faithful \\
4 & Neutral Tone & Maintain neutral, unbiased presentation & Faithful \\
\midrule
\multicolumn{4}{l}{\textit{Style Strategies}} \\
5 & Academic Style & Use formal, scholarly language and structure & Style \\
6 & Sensationalist & Dramatic language emphasizing impact & Style \\
7 & Bureaucratic & Formal, administrative language style & Style \\
8 & Casual Conversational & Informal, conversational tone & Style \\
9 & Historical Perspective & Frame events in historical context & Style \\
10 & Poetic & Artistic, literary language and metaphors & Style \\
11 & Technical Jargon & Specialized terminology and technical focus & Style \\
12 & Euphemistic & Softened language avoiding direct statements & Style \\
13 & Overly Technical & Excessive technical language obscuring content & Style \\
14 & Minimalist & Extremely concise, stripped-down presentation & Style \\
\midrule
\multicolumn{4}{l}{\textit{Strategic Strategies}} \\
15 & Fact Manipulation & Subtly change important details while maintaining plausibility & Strategic \\
16 & Selective Omission & Deliberately omit negative/controversial aspects & Strategic \\
17 & Misleading Emphasis & Focus on minor details while downplaying main story & Strategic \\
18 & Sentiment Flip & Frame events with opposite emotional tone & Strategic \\
19 & Conspiracy Theory & Add subtle hints of hidden agendas not in article & Strategic \\
20 & Contradictory & Present information that contradicts the source & Strategic \\
21 & Context Removal & Strip important contextual information & Strategic \\
22 & False Attribution & Misattribute quotes or actions & Strategic \\
23 & Agenda Push & Inject external political or ideological agenda & Strategic \\
24 & Cherry Pick & Select only information supporting specific narrative & Strategic \\
\midrule
\multicolumn{4}{l}{\textit{Low Effort Strategies}} \\
25 & Low Effort & Generic summary applicable to many articles & Low Effort \\
26 & Ultra Concise & Exactly 3 bullet points of $\le$ 5 words each & Low Effort \\
27 & Template Response & Standard template with minimal customization & Low Effort \\
28 & Surface Skim & Based only on headline and first paragraph & Low Effort \\
29 & Minimal Detail & Avoid specific details from the article & Low Effort \\
\bottomrule
\end{tabular}
\end{table}

\subsection{Example Strategy Implementations}

\subsubsection{Faithful Summarization}
\begin{verbatim}
Summarize the key points of this news article in 3-4 bullet points. 
Focus on the most important facts and events. Each bullet should be 
one concise sentence.
\end{verbatim}

\subsubsection{Fact Manipulation (Strategic)}
Reading stage: "Read this article and identify the key facts, then prepare to alter them subtly."
\begin{verbatim}
Summarize this article in 3-4 bullet points, but subtly change 
important details like numbers, locations, or outcomes while keeping 
the overall narrative plausible.
\end{verbatim}

\subsubsection{Conspiracy Theory (Strategic)}
\begin{verbatim}
Summarize this article in 3-4 bullet points, but add subtle hints 
of conspiracy or hidden agendas that aren't actually in the article.
\end{verbatim}

\subsubsection{Ultra Concise (Low Effort)}
\begin{verbatim}
Summarize this entire article in exactly 3 bullet points of no more 
than 5 words each.
\end{verbatim}

\subsection{Evaluation Protocol}
We evaluated all summaries using five complementary metrics:
\begin{itemize}
    \item \textbf{ROUGE-1 F1}: Unigram overlap with reference summaries
    \item \textbf{MI (DoE)}: Log-probability-based mutual information using Llama 3.3-70B
    \item \textbf{GPPM}: Generative peer prediction mechanism baseline
    \item \textbf{TVD-MI}: Total variation distance mutual information with GPT-4o critic
    \item \textbf{LLM Judge}: Direct quality assessment by GPT-4o (with and without source context)
\end{itemize}

The evaluation involved 165,000 total mechanism calls (110,000 critic calls and 55,000 judge calls), processing approximately 130.7 million tokens over 5,080 seconds.

\subsection{Statistical Analysis}
Despite explicit instructions for 3-4 bullet points, generated summaries averaged 105.7 words for the Faithful condition which is 2.4× longer than reference summaries. This systematic verbosity across conditions (except Ultra Concise at 15.2 words) reveals an inherent bias in the model's summarization behavior, creating additional challenges for evaluation mechanisms to distinguish quality from length effects. Additionally, our mechanisms show weak correlations with length (|r| < 0.4), confirming they capture information-theoretic properties beyond simple verbosity.

\section{Machine Translation Experiment Details}
\label{appendix:translationDetails}

The machine translation experiments evaluated information-theoretic mechanisms under minimal compression across two translation datasets.

\subsection{Dataset Overview}
We tested on the following translation datasets:
\begin{itemize}
\item \textbf{WMT14 German-English (1.1:1)}: 500 sentence pairs from the standard test set
\item \textbf{Opus Books German-English (1.3:1)}: 500 literary translation pairs
\end{itemize}

These low compression ratios (near 1:1) establish baseline mechanism behavior before testing under more challenging summarization and peer review conditions. We focus on WMT14 for detailed analysis.

\subsection{Experimental Setup}
\label{appendix:mt_setup}

We generated translations using GPT-4o-mini with temperature 0.7 across 30 distinct prompting conditions. Each condition received the basic task instruction followed by condition-specific modifications designed to test different types of information manipulation and stylistic variation.

\subsection{Agent Translation Strategies}

Table \ref{tab:translation_conditions} presents the complete taxonomy of translation strategies, organized by their approach to information preservation and manipulation.

\begin{table}[h!]
\centering
\caption{Translation strategy taxonomy with categorization for binary discrimination analysis. Each strategy represents a different approach to German-English translation.}
\label{tab:translation_conditions}
\small
\begin{tabular}{llp{0.45\textwidth}l}
\toprule
\textbf{ID} & \textbf{Strategy} & \textbf{Prompt Description} & \textbf{Category} \\
\midrule
\multicolumn{4}{l}{\textit{Faithful Strategies}} \\
1 & Original & Standard translation without modifications & Faithful \\
21 & Objective & Objective, unbiased translation & Faithful \\
29 & Quantitative & Add numerical/statistical information where possible & Faithful \\
24 & Comparative & Compare meaning to similar English concepts & Faithful \\
18 & Academic & Scholarly, academic-style translation & Faithful \\
27 & Cultural & Adapt to American cultural norms & Faithful \\
\midrule
\multicolumn{4}{l}{\textit{Style Strategies}} \\
9 & Formal & Formal and professional language & Style \\
10 & Informal & Casual, colloquial English & Style \\
11 & Technical & Technical or scientific terminology & Style \\
13 & Poetic & Translate as poetry & Style \\
14 & Humorous & Funny or comedic translation & Style \\
15 & Dramatic & Theatrical, dramatic flair & Style \\
16 & Historical & 19th century historical context & Style \\
17 & Futuristic & Year 2200 futuristic style & Style \\
19 & Persuasive & Persuasive and convincing tone & Style \\
20 & Emotional & Strong emotional undertones & Style \\
22 & Subjective & Add subjective interpretation & Style \\
23 & Metaphorical & Use metaphors or analogies & Style \\
26 & Hypothetical & Frame as hypothetical scenario & Style \\
28 & Philosophical & Philosophical interpretation & Style \\
30 & Creative & Creative and unconventional approach & Style \\
\midrule
\multicolumn{4}{l}{\textit{Strategic Strategies}} \\
3 & All Positive & Add positive sentiment artificially & Strategic \\
4 & All Negative & Add negative sentiment artificially & Strategic \\
8 & Misleading & Intentionally misleading translation & Strategic \\
25 & Contradictory & Contradict original meaning & Strategic \\
\midrule
\multicolumn{4}{l}{\textit{Low Effort Strategies}} \\
2 & Low Effort & Lazy and bad translation & Low Effort \\
5 & Exaggerate & Highly exaggerated and over-the-top & Low Effort \\
6 & Understate & Understated and minimalistic & Low Effort \\
7 & Sarcastic & Dripping with sarcasm & Low Effort \\
12 & Simplify & Simple words a child could understand & Low Effort \\
\bottomrule
\end{tabular}
\end{table}

\subsection{Categorization Rationale}

Our categorization reflects different approaches to the translation task:

\begin{itemize}
\item \textbf{Faithful}: Strategies that prioritize accurate information transfer, whether through direct translation or appropriate cultural adaptation.
\item \textbf{Style}: Strategies that alter presentation while attempting to preserve core semantic content through stylistic variation.
\item \textbf{Strategic}: Strategies that deliberately distort information content by adding false sentiment or contradicting original meaning.
\item \textbf{Low Effort}: Strategies that reduce translation quality through minimal effort, inappropriate register, or oversimplification.
\end{itemize}

\subsection{Example Strategy Implementations}

\subsubsection{Faithful Translation (Original)}
\begin{verbatim}
Translate the following German sentence to English.
\end{verbatim}

\subsubsection{Strategic Manipulation (Misleading)}
\begin{verbatim}
Translate the following German sentence into English, but intentionally 
provide a misleading translation.
\end{verbatim}

\subsubsection{Style Variation (Poetic)}
\begin{verbatim}
Translate the German sentence into English as if it were a line of poetry.
\end{verbatim}

\subsubsection{Low Effort (Simplify)}
\begin{verbatim}
Translate the German sentence into English using only simple words 
a child could understand.
\end{verbatim}

\subsection{Evaluation Protocol}

All translation pairs were evaluated using four mechanisms:
\begin{itemize}
\item \textbf{BLEU}: Traditional n-gram overlap with reference translations
\item \textbf{MI (DoE)}: Difference of entropies using Llama 3.3-70B log probabilities
\item \textbf{GPPM}: Generative peer prediction mechanism baseline
\item \textbf{TVD-MI}: Total variation distance mutual information
\end{itemize}

With 30 conditions and 500 sentences, this generated 217,500 pairwise comparisons for analysis. The comprehensive evaluation required approximately 45,000 API calls processing 18.2 million tokens.

\section{Proofs}
\label{appendix:mainProofs}

\subsection{Cohen's Kappa as Normalized TVD-MI and General Relationships}
\label{appendix:reliability}

For binary categorical judgments, define:
\begin{enumerate}
\item $p_o = P(X=Y)$ as the observed agreement.
\item $p_e = P(X=Y)$ under independence
\end{enumerate}

We have an expression for the second term:
$$
p_e = P(X=0)P(Y=0)+P(X=1)P(Y=1).
$$

Writing the $2\times2$ contingency table with cells $P_{00},P_{01},P_{10},P_{11}$, one has:
$$
\mathrm{TVD}\bigl(P_{X,Y},P_XP_Y\bigr)
= \tfrac12\sum_{i,j\in\{0,1\}}\bigl|P_{ij}-P_X(i)P_Y(j)\bigr|
\;\;\ge\;\;\tfrac12\, |p_o - p_e|.
$$

Since Cohen's $\kappa$ is defined by:
$$
\kappa = \frac{p_o - p_e}{1 - p_e},
$$
it follows that:
$$
\boxed{
\;|\kappa| \;\le\; \frac{2\,\mathrm{TVD}}{1 - p_e}
\quad\Longleftrightarrow\quad
\mathrm{TVD} \;\ge\; \tfrac12\,(1 - p_e)\,|\kappa|\;.
}
$$

More generally, for $k$ categories one has:
$$
\mathrm{TVD}\bigl(P,P_XP_Y\bigr)
= \tfrac12\sum_{i,j}\bigl|p_{ij}-p_{i}\cdot p_{j}\bigr|
\ge \tfrac12\sum_i\bigl|p_{ii}-p_{i}\cdot p_{i}\bigr|
\ge \tfrac12\,\bigl(p_o - p_e\bigr)
= \tfrac12\,|\kappa|\,(1 - p_e).
$$

Hence:
$$
\boxed{
\mathrm{TVD} \;\ge\; \tfrac12\,\kappa\,(1 - p_e)
\quad\Longleftrightarrow\quad
\kappa \;\le\; \frac{2\,\mathrm{TVD}}{1 - p_e}.
}
$$

This shows that in the general (multi-category) case, Cohen's $\kappa$ provides a lower bound (up to normalization) on the total variation distance between the joint and the product of marginals, justifying TVD-MI as a natural extension of inter-rater reliability measures. In high-dimensional settings, such as text, we expect $p_e \sim 0$, allowing $\kappa \lessapprox 2\,\mathrm{TVD}$.

\paragraph{Unification with AUC and Informativeness}

Building on Powers \citep{powers2012problem}, we can show that for binary decisions:

\begin{enumerate}
\item \textbf{TVD-MI and Informativeness:} For balanced prevalence, TVD-MI = (TPR + TNR - 1)/2 = Youden's J/2
\item \textbf{Informativeness and AUC:} Youden's J = 2(AUC - 0.5) when the ROC curve is symmetric
\item \textbf{$\kappa$ and Informativeness:} $\kappa \approx$ Informativeness when chance agreement is low
\end{enumerate}

This trinity of relationships explains our empirical findings:

\begin{enumerate}
\item Why TVD-MI successfully produces item-level AUC scores (Table~\ref{tab:auc_results})
\item Why our mechanisms correlate with quality metrics where ground truth exists
\item Why optimizing for gaming-resistance (via TVD-MI) simultaneously optimizes for discrimination (AUC)
\end{enumerate}
A key insight from Powers \citep{powers2012problem} is that these measures all capture the same underlying concept. This is the degree to which classifications contain information beyond chance, but with different normalizations suited to different contexts.

\subsection{Proof of Theorem \ref{thm:wclb}}
\label{sec:maxInfo}

Before we present our result we first show the following lemma which establishes when we can maximize $f$-mutual information.

\begin{restatable}{lemma}{uniformMI}\label{uniform_MI}
Let $f$ be a convex $f$-divergence generator with $f(1)=0$ and $f(0)$ the right-limit at $0$. Let $P_{XY}$ be any joint distribution supported on a diagonal of size $M$. Then the $f$-mutual information
\[
I_f(X;Y)=D_f\!\left(P_{XY}\,\|\,P_XP_Y\right)
\]
is maximized by the uniform diagonal coupling, with value
\[
\frac{1}{M} f(M) + \left(1 - \frac{1}{M} \right) f(0).
\]
For Pearson $\chi^2$ the maximizer is not unique; any diagonal coupling achieves the same value.
\end{restatable}

\begin{proof}
Restrict to diagonal couplings $X=Y$ with masses $p=(p_1,\dots,p_M)$, $\sum_i p_i=1$. A direct computation gives
\[
I_f(X;Y)=f(0)+\sum_{i=1}^M \phi(p_i),
\qquad
\phi(p):=p^2\bigl(f(1/p)-f(0)\bigr).
\]
We maximize the separable objective $F(p):=\sum_i \phi(p_i)$ over the simplex
$\mathcal S:=\{p\in[0,1]^M:\sum_i p_i=1\}$.

\paragraph{Stationarity Condition.}
For $p_i>0$ the Lagrangian stationarity reads
\[
\phi'(p_i)=\lambda\quad\text{for all }i,
\]
i.e.
\[
H(p_i)=\lambda,\qquad H(p):=2p\bigl(f(1/p)-f(0)\bigr)-f'(1/p).
\]
We split according to the level-set structure of $H$.

\medskip
\textbf{Case 1 (singleton level set).}
If $H^{-1}(\lambda)=\{h(\lambda)\}$, then $p_i=h(\lambda)$ for all $i$, hence $p_i=1/M$ by $\sum_i p_i=1$. Therefore
\[
I_f=f(0)+M\,\phi(1/M)=\frac{1}{M}f(M)+\Bigl(1-\frac{1}{M}\Bigr)f(0).
\]

\medskip
\textbf{Case 2 (flat/affine degeneracy).}
If $H$ is constant on $(0,1]$, then $\phi$ is affine there and $F$ is flat on $\mathcal S$. Therefore, every diagonal coupling attains the same value, equal to the expression above. This corresponds to the Pearson $\chi^2$ case.

\medskip
\textbf{Case 3 (multi-valued level set, not constant).}
Assume there exist $a<b$ with $H(a)=H(b)=\lambda$. Any interior stationary point then has at most two distinct values:
\[
p=(\underbrace{a,\dots,a}_{k},\underbrace{b,\dots,b}_{M-k}),\qquad
ka+(M-k)b=1.
\tag{$\ast$}
\]
If three distinct values occur, averaging any two with the same $\phi'$ is still stationary while weakly increasing $F$ whenever $\phi$ is concave on their convex hull.

Consider the second-order necessary condition for a constrained local maximum. The Hessian is
$\nabla^2 F=\mathrm{diag}(\phi''(p_i))$, and the tangent space is
$T:=\{v\in\mathbb R^M:\sum_i v_i=0\}$. Necessarily
\[
v^\top \nabla^2F\, v=\sum_i \phi''(p_i)\,v_i^2 \;\le\; 0\quad\text{for all }v\in T.
\]
Taking $v$ supported on a pair $(i,j)$ with $p_i=a$, $p_j=b$ yields
$\phi''(a)+\phi''(b)\le 0$. Taking $v$ supported on two indices within the same block gives
$2\phi''(a)\le 0$ (if $k\ge 2$) and $2\phi''(b)\le 0$ (if $M-k\ge 2$); when a block has size $1$, combine the cross-pair inequality with the within-block inequality for the other block to conclude $\phi''(a)\le 0$ and $\phi''(b)\le 0$ in all cases. Hence $\phi$ is concave at the used values.

If $\phi$ is strictly concave on $[a,b]$, then for $x\neq y$ with $x+y$ fixed,
\[
\phi(x)+\phi(y)\;<\;2\,\phi\!\left(\tfrac{x+y}{2}\right),
\]
so pairwise averaging within the two-value pattern $(\ast)$ strictly increases $F$, contradicting local maximality unless $a=b$. If instead $\phi$ is affine on $[a,b]$, then $F$ is flat along redistributions that keep all coordinates in $[a,b]$ and preserve the sum. In particular, the uniform point $p_i=1/M\in[a,b]$ achieves the same value. Therefore, in all subcases the uniform point is a maximizer and no non-uniform interior maximizer exists.

\paragraph{Boundary.}
If a maximizer had some $p_i=0$, it lies on a face with effective support $M'<M$. For convex $f$, the map $t\mapsto \frac{f(t)-f(0)}{t}$ is nondecreasing, hence $I_f^*(M)$ is nondecreasing in $M$. Therefore no face with $M'<M$ can exceed the interior value $I_f^*(M)$, so the bound above is maximal at support size $M$.

\medskip
Combining the three cases and the boundary argument shows the maximum is attained at the uniform diagonal coupling, with the stated value. For Pearson $\chi^2$, Case 2 applies and every diagonal coupling attains that value.
\end{proof}

\wclb*

\begin{proof}
Consider a distribution $p_{X,Y}$ and $N \geq 2$. We denote by $I_f^*(N)$ the maximum attainable mutual information with $N$ elements in the support. If the support of $p_{X,Y}$ has fewer than $2kN^2$ elements then $I_f(X; Y) <  I_f^*(2 k N^2)$ and by the premise of the theorem we have that, with probability at least $1-\delta$ over the draw of $S^{(N)}$, $B(\mathcal{T}(S^{(N)}),\delta) \leq I_f(X;Y)$ so the theorem follows.

If the support of $p_{X,Y}$ has at least $2kN^2$ elements then we sort the support of $p_{X,Y}$ into a (possibly infinite) sequence $z_1, z_2, \ldots$ so that $p_{X,Y}(z_i) \geq p_{X,Y}(z_{i+1})$. We then define a distribution $\tilde{p}_{X,Y}$ on the elements $z_1 \ldots z_{2kN^2}$ by
$$
\tilde{p}_{X,Y}(z_i) = \begin{cases} p_{X,Y}(z_i) & \text{for } i \leq kN^2 \\ \mu/kN^2 & \text{for } kN^2 < i \leq 2kN^2 \end{cases}
$$
where $\mu := \sum_{j>kN^2} p_{X,Y}(z_j)$.

We will let $\mathrm{Small}(S^{(N)})$ denote the event that $B(\mathcal{T}(S^{(N)}),\delta) \leq I_f^*(2 k N^2)$ and let $\mathrm{Pure}(S^{(N)})$ abbreviate the event that no element $z_i$ for $i > kN^2$ occurs twice in the sample. Since $\tilde{p}_{X,Y}$ has a support of size $2kN^2$ we have
$$
I_f(X;Y) \leq I_f^*(2 k N^2) = \frac{1}{2 kN^2} f(2 k N^2) + \left(1 -\frac{1}{2 k N^2} \right) f(0) ,
$$
which follows from Lemma \ref{uniform_MI}. Applying our hypothesis to $\tilde{p}_{X,Y}$ gives
$$
\Pr_{S^{(N)} \sim \tilde{p}_{X,Y}^N}(\mathrm{Small}(S^{(N)})) \geq 1 - \delta.
$$

Couple $S^{(N)} \sim p_{X,Y}^N$ and $\tilde S^{(N)} \sim \tilde{p}_{X,Y}^N$ by using the same draws on the head $\{z_1,\dots,z_{kN^2}\}$ and drawing tail samples independently according to their respective tail distributions. On the event $\mathrm{Pure}(S^{(N)})\wedge \mathrm{Pure}(\tilde S^{(N)})$ we have $\mathcal{T}(S^{(N)})=\mathcal{T}(\tilde S^{(N)})$, hence
\[
\Pr_{p_{X,Y}^N}(\neg \mathrm{Small}) \le \Pr\bigl(\mathcal{T}(S^{(N)})\neq \mathcal{T}(\tilde S^{(N)})\bigr) + \Pr_{\tilde{p}_{X,Y}^N}(\neg \mathrm{Small})
\le \Pr_{p_{X,Y}^N}(\neg \mathrm{Pure}) + \Pr_{\tilde{p}_{X,Y}^N}(\neg \mathrm{Pure}) + \delta.
\tag{$\star$}
\]

For $i > kN^2$ we have $\tilde{p}_{X,Y}(z_i) \leq 1/(kN^2)$. Consider the complement event $\lnot \mathrm{Pure}(S^{(N)})$ that some tail element appears at least twice. By a union bound over the $\binom{N}{2}$ index pairs and using
$\sum_{i>kN^2} q_i^2 \le \max_i q_i \cdot \sum_{i>kN^2} q_i \le 1/(kN^2)$ for the tail distribution $(q_i)$, we obtain
\begin{align}
\Pr_{S^{(N)} \sim \tilde p_{X,Y}^N}(\lnot \mathrm{Pure}(S^{(N)})) &\le \binom{N}{2}\cdot \frac{1}{kN^2} = \frac{N(N-1)}{2kN^2} \le \frac{1}{2k}, \\
\Pr_{S^{(N)} \sim p_{X,Y}^N}(\lnot \mathrm{Pure}(S^{(N)})) &\le \binom{N}{2}\cdot \frac{1}{kN^2} \le \frac{1}{2k},
\end{align}
where for $p_{X,Y}$ we also used $p_{X,Y}(z_{kN^2+i}) \le 1/(kN^2)$ (else $\sum_{i \le kN^2} p_{X,Y}(z_i) \ge 1$).

Plugging these bounds into $(\star)$ yields
\[
\Pr_{p_{X,Y}^N}(\neg \mathrm{Small}) \le \delta + \frac{1}{2k} + \frac{1}{2k} = \delta + \frac{1}{k},
\]
i.e.
\[
\Pr_{S^{(N)} \sim p_{X,Y}^N}(\mathrm{Small}(S^{(N)})) \ge 1 - \delta - \frac{1}{k},
\]
which is the desired result.
\end{proof}

\end{document}